%% file: paper_for_arxiv.tex
\documentclass[twocolumn]{article}
\usepackage[margin=48pt]{geometry}

\usepackage{times}  
\usepackage{helvet} 
\usepackage{courier}
\usepackage[hyphens]{url} 
\usepackage{graphicx} 
\usepackage{natbib} 
\usepackage{caption}
\DeclareCaptionStyle{ruled}{labelfont=normalfont,labelsep=colon,strut=off} 
\usepackage{amsmath}

\usepackage{xcolor}
\usepackage{lineno}
\usepackage{verbatim}

\usepackage{amssymb}
\usepackage{amsthm}
\usepackage[linesnumbered,ruled,vlined,algo2e]{algorithm2e}
\usepackage{algorithm}
\usepackage{caption}
\usepackage{graphicx}
\usepackage{sidecap}
\usepackage{subfigure}

\newtheorem{theorem}{Theorem}
\newtheorem{lemma}{Lemma}

\newenvironment{talign}
 {\align}
 {\endalign}

\input{notation}

\title{Training-Time Attacks against \(k\)-Nearest Neighbors}
\author{Ara Vartanian$^1$ \and Will Rosenbaum$^2$ \and Scott Alfeld$^2$}
\date{%
  $^1$\texttt{aravart@cs.wsic.edu} \\
  University of Wisconsin--Madison\\%
  $^2$\texttt{\{wrosenbaum, salfeld\}@amherst.edu}\\
  Amherst College\\[2ex]%
    \today
}

\begin{document}

\maketitle

\begin{abstract}
  \input{abstract}

\end{abstract}

\section{Introduction}
\label{sec:intro}
\input{intro}

\section{Problem Setup}
\label{sec:setup}
\input{setup}

\section{Proof of Hardness}
\label{sec:hardness}
\input{hardness}

\section{Proposed Methods}
\label{sec:methods}
\input{methods}

\section{Experiments}
\label{sec:exp}
\input{exp}

\section{Related Work}
\label{sec:relwork}
\input{relwork}

\section{Conclusions} %
\label{sec:conc}
\input{conc}

\vfill

\pagebreak

~

\vfill

\pagebreak

\bibliographystyle{apalike}
\ifsqueeze
    {\fontsize{10.5}{\lineskip}\selectfont \bibliography{bibliography.bib}}
\else
    \bibliography{bibliography.bib}
\fi

\ifdraft

\else \fi

\vfill

\pagebreak

\ifincludeapp

\input{appendix}

\fi
\end{document}

%% file: notation.tex
\newif\ifdraft
\drafttrue

\newif\ifincludeapp
\includeapptrue

\newif\ifsqueeze
\squeezefalse

\ifsqueeze
\newcommand{\defaultfigsize}{.80\columnwidth}
\else
\newcommand{\defaultfigsize}{1.0\columnwidth}
\fi

\newcommand{\trset}{\mathcal{D}}

\newcommand{\tlr}{\texttt{TLR}}
\newcommand{\tsi}{\texttt{TSI}}
\newcommand{\atkr}{\texttt{ATKR}}
\newcommand{\evalpool}{\mathcal{T}}

\newcommand{\atkrset}{\mathbf{\Delta}}

\newcommand{\evalx}{\x^{\text{ev}}_i}
\newcommand{\evaly}{y^{\text{ev}}_i}

\newcommand{\evalitem}{(\evalx, \evaly)}

\newcommand{\IR}{\mathrm{IR}}
\newcommand{\labeledballs}{\mathcal{B}^{\mathrm{lab}}}

\newcommand{\atkrbudget}{b}
\newcommand{\atkrtimebudget}{T}
\newcommand{\atkrlabelset}{\mathcal{C}}
\newcommand{\score}{\mathrm{score}}
\newcommand{\coverage}{\mathrm{coverage}}

\newcommand{\potentialedge}{\mathbf e}
\newcommand{\potentialedgeelem}{\ball}

\newcommand{\learner}{\mathcal{A}}

\newcommand{\hypergraph}{G}

\newcommand{\chg}{\texttt{CHOPPA}}

\newcommand{\greedyatk}{\(\texttt{GIT}^2\texttt{ACHOPPA}\)}

\newcommand{\atkknn}{\textsc{ATK-KNN}}
\newcommand{\maxis}{\textsc{MAX-IS}}
\newcommand{\maxint}{\textsc{MAX-INT}}

\newcommand{\abs}[1]{\left|#1\right|}
\newcommand{\norm}[1]{\left\|#1\right\|}
\newcommand{\set}[1]{\left\{#1\right\}}
\newcommand{\sucht}{\,\middle|\,}
\newcommand{\eps}{\varepsilon}
\renewcommand{\th}{{}^{\mathrm{th}}}
\newcommand{\paren}[1]{\left(#1\right)}

\newcommand{\trainingset}{\trset}

\newcommand{\calS}{\mathcal{S}}

\newcommand{\ir}{\mathrm{IR}}

\DeclareMathOperator{\nbr}{nbr}
\DeclareMathOperator{\lbl}{lbl}
\DeclareMathOperator{\plurality}{plurality}

\newcommand{\opt}{\mathrm{opt}}
\newcommand{\cov}{\mathrm{cov}}
\newcommand{\atkopt}{\atkrset_\opt}
\newcommand{\atkcov}{\atkrset_\cov}

\newcommand{\dball}{$d$-ball}

\newcommand{\ball}{B}
\newcommand{\balls}{\mathcal{B}}
\newcommand{\dimn}{d}             %
\newcommand{\basis}{\mathbf{e}}   %
\newcommand{\pt}{\mathbf{p}}      %
\newcommand{\x}{\mathbf{x}}       %
\newcommand{\bfc}{\mathbf{c}}
\newcommand{\R}{\mathbf{R}}       %
\newcommand{\Rd}{\R^\dimn}
\newcommand{\radius}{r}

\newcommand{\graph}{G}
\newcommand{\vertices}{V}
\newcommand{\edges}{E}

\renewcommand{\th}{{}^{\mathrm{th}}}

\SetKwInput{KwInput}{Input}                %
\SetKwInput{KwOutput}{Output}              %

%% file: abstract.tex
Nearest neighbor-based methods are commonly used for classification tasks and as subroutines of other data-analysis methods. 
An attacker with the capability of inserting their own data points into the training set can manipulate the inferred nearest neighbor structure.
We distill this goal to the task of performing a training-set data insertion attack against \(k\)-Nearest Neighbor classification (\(k\)NN).
We prove that computing an optimal training-time (a.k.a.~poisoning) attack against \(k\)NN classification is NP-Hard, even when \(k = 1\) and the attacker can insert only a single data point.
We provide an anytime algorithm to perform such an attack, and a greedy algorithm for general \(k\) and attacker budget.
We provide theoretical bounds and empirically demonstrate the effectiveness and practicality of our methods on synthetic and real-world datasets.
Empirically, we find that \(k\)NN is vulnerable in practice and that dimensionality reduction is an effective defense.
We conclude with a discussion of open problems illuminated by our analysis.

%% file: intro.tex
Using machine learning (ML) and automated data analysis methods in practice introduces security vulnerabilities.
An attacker can assert their limited influence over the input data to manipulate the output of the learning system.
\ifsqueeze \color{black}
Here, we consider training-time %
 (a.k.a.~poisoning) attacks against nearest-neighbor based classification, where the attacker perturbs data prior to learning.
\color{black} \else \color{black}
Attacks can be categorized based on when in the process the attacker asserts their influence.
In training-time (a.k.a.~poisoning) attacks, the attacker perturbs data prior to learning.
In deployment-time (a.k.a.~evasion) attacks, the attacker instead perturbs test instances.
Here, we consider training-time attacks against nearest-neighbor based classification.
\color{black} \fi

\(k\)-Nearest Neighbors (\(k\)NN) is a classical non-parametric ML algorithm,
where 
a data point label is chosen by selecting the plurality class of its \(k\)
closest neighbors in the training set.
Outside of classification, many algorithms construct a nearest neighbor graph (connecting each data point to its closest neighbors) as a subroutine.
Examples include ISOMAP~(\citet{isomap}), persistent homology~(\citet{zhu2013persistent,zhu2016stochastic}), and various spectral clustering algorithms~(\citet{spectralclustering}).

We study neighbor-based classification as the canonical process of constructing nearest-neighbor graphs.
An attacker which can affect classifications can similarly alter the nearest-neighbor graph.
\ifsqueeze \color{black}
\color{black} \else \color{black}
For example, consider a facial-recognition system which, given a photograph, reports the \(k\) most similar photographs in the database.
An agent might insert manufactured photographs into the facial-recognition system's database with the intent of having images of her classified as someone else.
\color{black} \fi
We lay the foundation for a theoretical understanding of training-time attacks against \(k\)NN.
The primary contributions are:
\begin{enumerate}
\item We prove that computing an optimal training-time attack against \(k\)-Nearest Neighbor classification is NP-Hard.

\item We present a dynamic programming anytime algorithm that, if run to completion, computes an optimal single-point attack against $k$NN. Using this procedure as a subroutine, we employ and provide bounds for a greedy approach to efficiently construct an effective attack of any size against $k$NN.

\item%
 We perform an empirical investigation on a method of defense against our attack.
For example, reducing the number of dimensions via PCA in the MNIST dataset decreased the attacker's average effectiveness by as much as 26\% while increasing prediction error by less than 1\%.
\end{enumerate}

In Section~\ref{sec:setup}, after defining the learner and threat models for attacking \(k\)NN, we translate the attacker's task into a geometric problem.
 In Section~\ref{sec:hardness}, we prove that computing the optimal attack is NP-Hard, even for an attacker who targets one class with one additional data point. %
 In Section~\ref{sec:methods}, we present a (worst-case exponential-time) algorithm for finding the optimal single-point attack against $k$NN that uses dynamic programming to find successively better attacks as it proceeds, allowing it to be used as an anytime algorithm.
We then employ a greedy algorithm to construct a data-insertion attack of any size against $k$NN.
We show a bound for our greedy algorithm which is dependent on \(k\) (but independent of the attacker's budget).
We discuss the ``plug-and-play'' nature of our analysis, in that any future improvements to the single-point attacks algorithm automatically yield a better bound for our general attack.
In Section~\ref{sec:exp}, we empirically demonstrate the effectiveness and practicality of our attack on synthetic and real world data.
We then turn our attention to defense, and empirically investigate defense strategies based on dimensionality reduction.
We discuss related work in Section~\ref{sec:relwork} and conclude in Section~\ref{sec:conc} with a discussion of open problems.

%% file: setup.tex
We focus on the role of an attacker, inserting training data so as to manipulate what is learned.
\ifsqueeze \color{black}
\color{black} \else \color{black}
In what follows we present the learner and threat models for attacking \(k\)NN classification.
In Section~\ref{sec:geoview} we translate the task of mounting a single-point attack against \(k\)NN to a geometric problem.
This geometric interpretation serves as the basis of our proof of hardness as well as the generic attack we present in the following sections.
\color{black} \fi
\ifsqueeze \color{black}
A summary of notation used in this work is presented in Table~\ref{table:notation} in the appendix.
\color{black} \else \color{black}
A summary of notation used in this work is presented in Table~\ref{table:notation}.
\color{black} \fi

\subsection{Learner Model}

We consider a learning algorithm \(\learner\) performing \(k\)-Nearest Neighbors (\(k\)NN) classification.  Given a training set \(\trset\) of \(|\trset|\) points in \(d\) dimensions and query point \(x\), the learner predicts
\begin{align}
 \learner[x ; \trset] = \text{Plurality label of the } k~\text{points in \(\trset\) closest to \(x\)} \label{eqn:learner_defn} \nonumber
\end{align}
where ``closest'' is defined in terms of an \(\ell_p\) norm for \(p \in (1, \infty)\). %
The specific implementation is not important to the work presented here.
For simplicity of exposition, we ignore cases involving ties (two or more points in \(\trset\) are exactly the same distance from the query point) as well as cases where no strict plurality exists.
Our methods can handle such cases (using e.g., random tie breaking) with minor modifications.

\subsection{Threat Model}

\(k\)NN is a non-parametric learner.
As such, we consider a score function for the attacker based on the number of prescibed points $\learner$ classifies according to the attacker's wishes. This measure of the effectiveness of an attack is in contrast to most previous work on training-time attacks against parametric learners, where the effectiveness is defined in terms of a distance between (learned and target) models~(\citet{vorobeychik2018adversarial}).
We consider an attacker aiming to control the predicted labels of a set of points important to them.
For example they may have a collection of emails that she would like classified as ham.
They accomplish their task by constructing examples and inserting them into the training set.
We formalize this threat model as follows.

The attacker, \atkr, has full knowledge of the training set \(\trset\),the value of \(k\), and \(p\) (the norm used) .
\atkr~has a budget of \(\atkrbudget\) and their capability is to insert a set \(\atkrset\) of \(\atkrbudget\) points to \(\trset\), each with features of their choosing\footnote{
Constraints on which feature values \atkr~ can select, should any exist, can be incorporated into the attack algorithms as discussed in Section~\ref{sec:methods}.}
each with label we denote  as \(y^+\).

In addition, \atkr\ has an \emph{target pool}  \(\evalpool = \{(\x^{\text{ev}}_1, y^+), \ldots, (\x^{\text{ev}}_n, y^+)\}\). The \emph{score} of their attack is the number of points in $\evalpool$ that $\learner$ labels consistently with $\evalpool$:
\begin{equation}\textstyle
  \label{eqn:score_defn}
  \score(\atkrset, \evalpool, \trset, \learner) = \sum_{(\x, y) \in \evalpool} \mathbf{1}_{(\learner[\mathbf x ; \trset \cup \atkrset] = y)}.
\end{equation}
When $\evalpool, \trset$ and $\learner$ are clear from context, we denote this quantity by $\score(\atkrset)$. \atkr's goal is to find an attack $\atkrset$ that maximizes their score, i.e., to find %
 $\arg \max_{\atkrset \in (\Rd)^\atkrbudget} \score(\atkrset, \evalpool, \trset, \learner)$.

\input{example_circles_fig}

\subsection{Geometric View}
\label{sec:geoview}

We rephrase the problem of finding an optimal single-point training-time attack against \(k\)NN as the geometric problem of computing the maximum intersection of \(d\)-dimensional balls in $\Rd$.

We define the \emph{influencing region} (IR) of a point \(\evalitem\) in the target pool to be the set of feature vectors \(\x\) such that adding the element \((\x, y^+)\) to the training set with some multiplicity \(\ell\) changes the prediction of \(\evalx\) to \(y+\):
\begin{align}
  \scriptstyle
    \IR(\x, y, \trset) = \left\{\x' ~:~ \exists\ell ~|~ y = \learner\left[\x ; \trset \cup \{(\x', y^+)^\ell\} \right] \neq \learner\left[\x; \trset\right] \right\}
\end{align}
Here, we use \(\{(\x', y^+)^\ell\}\) to denote the multiset containing the element \((\x', y^+)\) with multiplicity \(\ell\). Observe that \(\IR(\x, y^+)\) is empty if \(\learner\left[\x ; \trset\right] = y^+\).

In the case of $k$-Nearest Neighbor classification using $\ell_p$ distance
each \(\IR(\x, y^+, \trset)\) is a (possibly empty) ball\footnote{Recall that a ball $\ball$ with center $\bfc$ and radius $\radius$ is defined as $\ball = \set{\x \in \Rd \sucht \norm{\x - \bfc}_p \leq \radius}$.} centered at \(\x\). Specifically, \(\IR(\x, y^+, \trset)\) can be computed as follows. Let \(\x_1, \x_2, \ldots, \x_k\) be \(\x\)'s \(k\) nearest neighbors in (strictly) increasing order of distance from \(\x\). Let \(j\) be the largest index such that setting the labels of \(\x_j, \x_{j+1}, \ldots, \x_{k}\) all to \(y^+\) changes \(\learner\)'s prediction of \(\x\) to \(y^+\). Then \(\IR(\x, y^+, \trset)\) is the ball centered at \(\x\) with radius \(r = \|\x - \x_j\|_p\).

To see that this this procedure correctly computes \(\IR(\x, y^+, \trset)\), first observe that by adding \((\x', y^+)\) to \(\trset\) with multiplicity \(\ell = k - j + 1\) and \(\|\x - \x'\|_2 < r\), these new points replace \(\x_j, \x_{j+1}, \ldots, \x_k\) among \(\x\)'s \(k\) nearest neighbors. Thus, \(y^+ = \learner\left[\x ; \trset \cup \{(\x', y^+)^\ell\}\right]\). Conversely, if \(\norm{\x - \x'}_p > r\), then adding \((\x', y^+)\) with any multiplicity cannot change the labels of \(\x\)'s \(j\) nearest neighbors. Therefore, by the maximality of \(j\), adding these points does not change \(\learner\)'s prediction of \(\x\).

Given a point \((\x, y^+) \in \evalpool\), we can associate a \emph{cost} and \emph{value} with \(\IR(\x, y^+, \trset)\). The \emph{cost} \(c = c(\x, y^+, \trset)\) is the minimum multiplicity of a point \((\x', y^+)\) with \(\x' \in \IR(\x, y^+, \trset)\) such that adding \((\x', y^+)\) with multiplicity \(c\) changes the classification of \(\x\) to \(y^+\). Note that for $k$NN, we always have \(c \leq \lceil k / 2\rceil\).\footnote{In the case where $y^+ = \learner\left[\x ; \trset\right]$ (hence $\IR(\x, y^+) = \varnothing$), we assign a cost of $c = 0$} The \emph{value} \(v\) is determined by
\begin{equation}
  \label{eqn:value}
v =
\begin{cases}
  1 &\text{if } y^{+} \neq \learner[\x; \trset] \\
  0 &\text{otherwise}.
\end{cases}
\end{equation}
Thus, \(v\) is the change score of an attack upon adding \((\x', y^+)\) to \(\atkrset\) with multiplicity \(c\), while \(\atkr\)'s score is unchanged when adding \((\x', y^+)\) with any multiplicity less than \(c\). Given \(\trset\) and \(\evalpool\), Algorithm~\ref{alg:construct_hyperspheres} constructs a family \(\labeledballs = \{((\x, r), v, c)\}_i\) consisting of the influencing regions of points in \(\evalpool\), together with their associated values and costs.

\input{construct_hyperspheres_pseudocode}

We define the \emph{total score increase} (\(\tsi\)) of an attack point \((\x, y^+)\) with multiplicity \(\ell\) to be the increase in \atkr's score if $(\x, y^+)$ is added to $\atkrset$ with multiplicity \(\ell\). $\tsi$ can be computed directly from $\labeledballs$:
\begin{equation}
  \label{eqn:tsi_defn}
  \tsi(\x, y^+, \ell, \labeledballs) =
    \sum_{((\x', r), v, c) \in \labeledballs} v \cdot \mathbf{1}_{\|\x - \x'\|_2 < r} \cdot \mathbf{1}_{\ell \geq c}
\end{equation}

This expression is central to the discussion of both our hardness result and our (greedy) algorithm for computing the full (\(\atkrbudget > 1\)) attack.
Figure~\ref{fig:example_circles} shows an example for \(p=2\) of constructing hyperspheres and the induced $\tsi$ values.

\ifsqueeze \color{black}
\color{black} \else \color{black}
\input{notation_table}

\color{black} \fi

%% file: example_circles_fig.tex
\begin{figure}[t]
  \centering
  \includegraphics[width=\defaultfigsize]{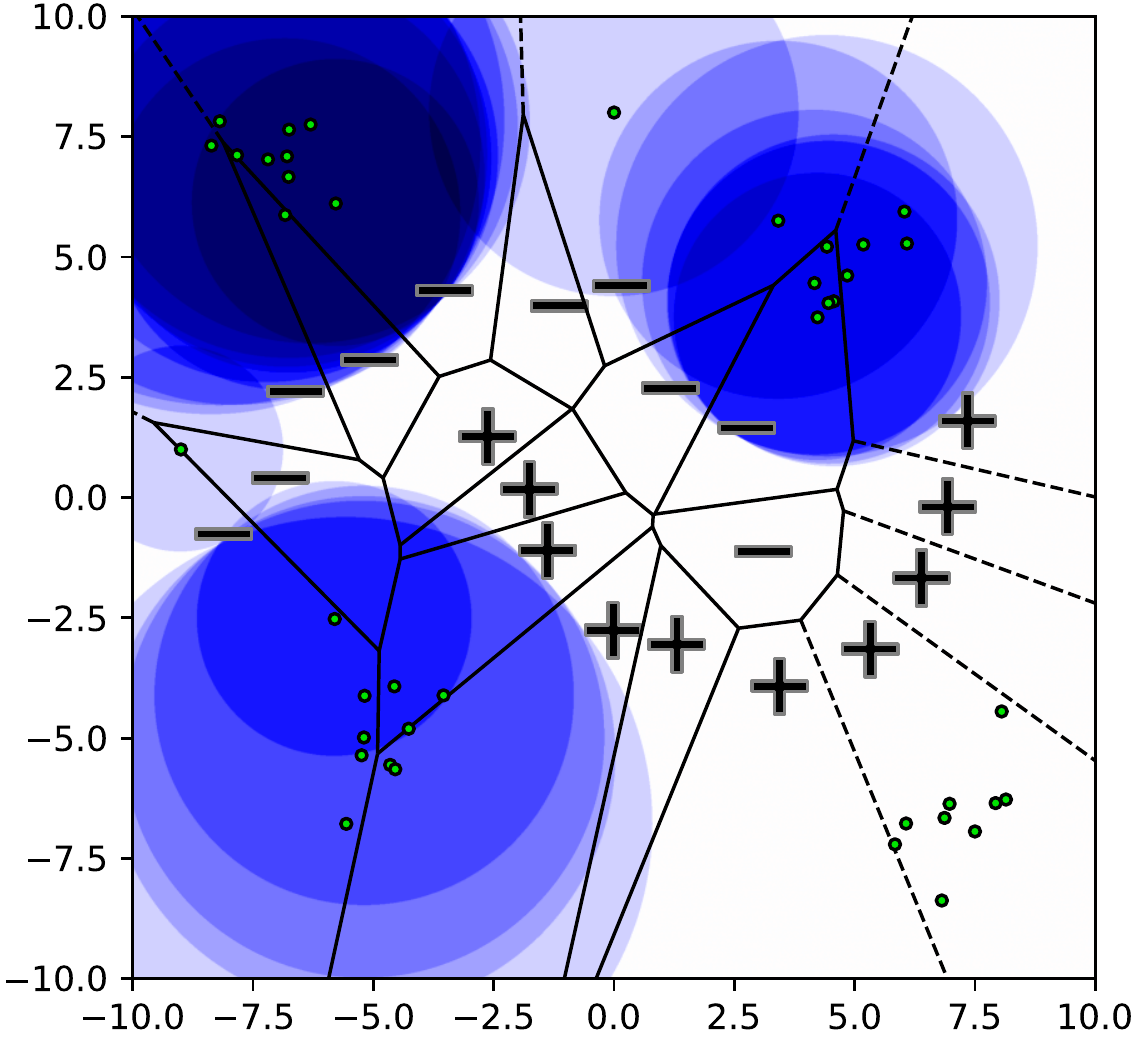}
 \caption{{\bf Synthetic Illustration.}
  An illustration the balls (in this case, disks) induced by \(\trset\) (the training set) and \(\evalpool\) (the target pool) for 1NN in \(d = 2\) dimensions.
  A training set from the well-known construction of two ``moons'' is shown as \(+\)'s and \(-\)'s, together with the decision boundary induced by it.
  We plot an example \(\evalpool\) as green dots.
  The color of each point in the plane indicates the total score increase ($\tsi$) associated with adding a single $+$ element at that point in the plane: darker shades of blue indicate larger $\tsi$ values.
  Each colored region is an intersection of disks centered at the points in $\evalpool$.
  The optimal single-point attack inserts a $+$ in one of the darkest blue regions.
  }
\label{fig:example_circles}
\end{figure}

%% file: construct_hyperspheres_pseudocode.tex
\begin{algorithm}[b]
  \small
  \DontPrintSemicolon
  \caption{\texttt{ConstructIRs}
   \setcounter{AlgoLine}{0}
  }
  \label{alg:construct_hyperspheres}
  \KwInput{\\
    \(\trset = \{(\x^{\text{tr}}, y^{\text{tr}})\}_i\), \(\evalpool = \{(\x^{\text{ev}}, y^{\text{ev}})\}_i, y^+\)
    }
  \KwOutput{\\
    \(\labeledballs = \{((\mathbf c, r), v, c)\}_i\) }
  \( \labeledballs \gets [~]\)
  
  \For{\((\x, y^{\text{ev}}) \in \evalpool\)}
      {
        \(\nbr \gets [\x_1, \x_2, \ldots, \x_k]\) where \(\x_i\) is \(\x\)'s \(i\th\) nearest neighbor in \(\trset\)
        
        \(\lbl \gets [y_1, y_2, \ldots, y_k]\) where \(y_i\) is \(\x_i\)'s label in \(\trset\)
        
        \(v \gets \) value according to~(\ref{eqn:value})
        
        \(j \gets k\)
        
        \While{\(\plurality(\lbl) \neq y^+\)}{
          \(\lbl[j] \gets y^+\)
          
          \(j \gets j - 1\)
          
        }
        Append \(((\x, \|\x - \x_{j+1}\|_2), v ,k - j + 1)\) to \(\labeledballs\)
      }
      \Return \(\labeledballs\)
\end{algorithm}

%% file: notation_table.tex
\begin{table}[t]
\small
  \centering
  \begin{tabular}{|c|l|}
    \hline

    Symbol                                      & Meaning\\ \hline \hline
    \(\atkr \)                                  & The attacker\\ \hline
    \(k \)                                      & (Of ``\(k\)NN'') number of neighbors considered\\ \hline
    \(\trset\)                                  & Learner's training set\\ \hline
    \(\evalpool\)                               & \atkr's evaluation pool\\ \hline
    \(d\)                                       & The ambient dimension of the data\\ \hline
    \(n \)                                      & Number of points in \(\evalpool\)\\ \hline
    \(\atkrset\)                              & The set of points \atkr~ adds to \(\trset\) %
    \\ \hline
    \(\atkrbudget \)                            & \atkr's (point) budget \\ \hline
    \(\atkrtimebudget \)                        &\atkr's time budget\\ \hline
    \(\learner \)                               & Learning algorithm\\ \hline %
    \(\learner[x; \trset] \)                    & Learner's prediction on \(x\) given training set \(\trset\)\\ \hline 
    \(\score\)                                  & \atkr's score for a given attack (\ref{eqn:score_defn})\\ \hline
    \tsi                                        & Total Score Increase (see (\ref{eqn:tsi_defn}))\\ \hline
    $\balls$ ($\labeledballs$)                   & A family of (labeled) balls in $\Rd$\\ \hline
    \(\mathbf 1_{\set{\cdot}}\)                      & 0-1 indicator function \\ \hline
  \end{tabular}

  \caption{{\bf Notation Summary.}}

    \label{table:notation}
  \end{table}

%% file: hardness.tex
We define the problem $\atkknn$ and show that it is NP-hard in general.
An instance of $\atkknn$ consists of a quadruple $(\trset, \evalpool, \atkrbudget, p)$, where $\trset, \evalpool \subseteq \Rd \times \atkrlabelset$, $\atkrbudget$ is a positive integer (\atkr's \emph{point budget}), and $p \geq 1$ is the $\ell_p$ norm used. Without loss of generality, we assume that all $(\x,y) \in \trset$ are labeled the same fixed label $y^-$, and \atkr\ wishes for $\learner$ to classify all $(\x, y^+) \in \evalpool$ as $y^+ \neq y^-$. The desired output is the maximum score achieveable by adding any $\atkrbudget$ element set $\atkrset \subseteq \Rd \times \atkrlabelset$ to $\trset$:

\begin{align}
\nonumber \atkknn(\trset, \evalpool, \atkrbudget, p) = \max_{\atkrset : \abs{\atkrset} = \atkrbudget} \score(\atkrset, \evalpool, \trset, k\text{NN}).
\end{align}

Note that \(\atkknn\) is a special case of attacking \(k\)NN.
Namely, \(k = b = 1\), and all points in the target pool are the same label.
In what follows we prove that \(\atkknn\) is NP-Hard.
In the appendix, we show that the general case (higher \(k, b\), and a heterogeneous target pool) is no easier.

\begin{theorem}
  \label{thm:hardness}
  $\atkknn$ is NP-hard.
\end{theorem}

When $\atkrbudget = \abs{\atkrset} = 1$, $\atkknn$ can be reduced to the problem of finding the maximum intersection of a family $\balls$ of \dball s in $\Rd$ using Algorithm~\ref{alg:construct_hyperspheres}. For $b > 1$ and $k = 1$, $\atkknn$ reduces to finding a set of $\atkrbudget$ points in $\Rd$ intersecting the maximum number of \dball s $\ball \in \balls$. %
We refer to the latter problem as the \emph{maximum intersection problem}, or $\maxint$. We show that even for $\atkrbudget = 1$, $\maxint$ is NP-hard. We then show that the instances of $\maxint$ constructed in our reduction can be realized as instances of $\atkknn$, thereby establishing Theorem~\ref{thm:hardness} for $k = 1$. We then give a simple modification of our construction that establishes the result for general $k$ and $\atkrbudget$ as well. We sketch the argument here; details appear in the appendix.

A $\dimn$-dimensional instance of $\maxint$ consists of a set $\balls = \set{\ball_1, \ball_2, \ldots, \ball_\ell}$, where each $\ball_i$ has a rational center and rational squared radius. The goal is to find the maximum number of mutually intersecting \dball s in $\balls$, i.e., to find $\max \big\{\abs{I} \,|\, \bigcap_{i \in I} \ball_i \neq \varnothing\big\}$.

We prove that $\maxint$ is NP-hard via a reduction from the maximum independent set ($\maxis$) problem. Recall that given a graph $\graph = (\vertices, \edges)$, an \emph{independent set} $I \subseteq \vertices$ is a set of vertices such that no pair of vertices in $I$ share an edge. $\maxis$ is to find the maximum cardinality among all independent sets in $G$. $\maxis$ is known to be NP-complete~\cite{Garey1979computers}.     %

\begin{theorem}
  \label{thm:max-int}
  $\maxint$ is NP-hard.
\end{theorem}

Towards proving Theorem~\ref{thm:max-int}, let $\graph = (\vertices, \edges)$ be a graph on $\dimn$ vertices. We denote $\vertices = \set{1, 2, \ldots, \dimn}$ and edges in $\edges$ by~$ij$. The idea of the reduction is to associate a ball $\ball_i \subseteq \Rd$ with each vertex $i \in \vertices$ and a ball $\ball_{ij}$ with each edge $ij \in \edges$. Each triple of the form $\ball_i$, $\ball_j$, $\ball_{ij}$ intersect pairwise, but the three do not mutally intersect (see Figure~\ref{fig:hardness}). The interpretation is that choosing a point $\x \in \ball_i$ corresponds to $i$ being in an independent set $I$, while $\x \in \ball_{ij}$ corresponds to the edge $ij$ being satisfied (i.e., $i$ and $j$ are not both in $I$). In our construction, a point $\x$ is contained in the largest number of balls in $\balls = \set{\ball_i} \cup \set{\ball_{ij}}$ if and only if the set of vertices $I = \set{i \sucht \x \in \ball_i}$ is a maximum independent set.

\begin{figure}[t]
  \centering
\ifsqueeze \color{black}
\includegraphics[width=0.5 \columnwidth]{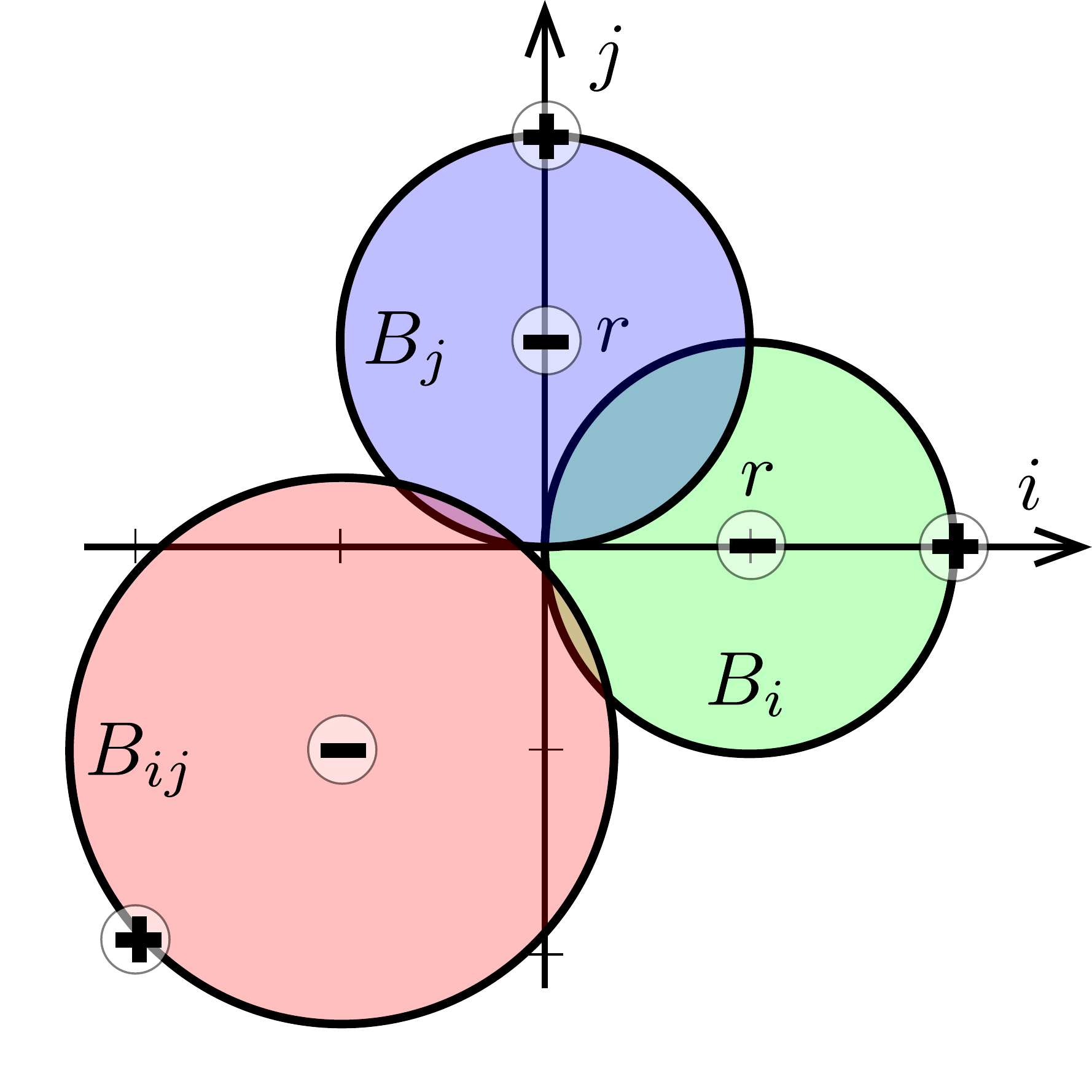}
\color{black} \else \color{black}
 \includegraphics[width=0.65 \columnwidth]{disks.pdf}
 \color{black} \fi
 \caption{The configuration $\balls$ for 2 dimensions. The disks intersect pair-wise, but not mutually. The instance of $\atkknn$ where points in $\trainingset$ are indicated with ``$+$'' signs, and points in $\evalpool$ are indicated with ``$-$'' signs realizes $\balls$.}
  \label{fig:hardness}
\end{figure}

Formally, we define the reduction $\varphi$ from $\maxis$ to $\maxint$ as follows. Given a graph $\graph$, we define a collection of balls $\varphi(G) = \balls = \set{\ball_i \sucht i \in \vertices} \cup \set{\ball_{ij} \sucht ij \in E}$ in $\Rd$, where $\ball_i$ has center $\radius \basis_i$ and radius $\radius$, while $\ball_{ij}$ has center $- \radius (\basis_i + \basis_j)$ and radius $\radius \sqrt[p]{2 - \eps}$. Here $\basis_i$ denotes the $i\th$ standard basis vector in $\Rd$. Values $\radius$ and $\eps$ are positive numbers whose values will be determined later. That is,
\begin{talign}
  \ball_i &= \big\{\x \in \Rd \big| \abs{x_i - \radius}^p + \sum_{k \neq i} \abs{x_k}^p \leq r^p\big\},\label{eqn:balli}\\
  \ball_{ij} &= \big\{\x \in \Rd \,\big|\, \abs{x_i + \radius}^p + \abs{x_j + \radius}^p + \sum_{k \neq i, j} \abs{x_k}^p \leq 2 \abs{r - \eps}^p\big\} \label{eqn:ballij}.
\end{talign}
The configuration $\balls$ is depicted in Figure~\ref{fig:hardness} for $\dimn = p =  2$. We view $\balls$ as a multiset where each $\ball_i$ occurs with multiplicity~1, and each $\ball_{ij}$ occurs with multiplicity~$\dimn$.

For any graph $\graph$, $\balls = \varphi(\graph)$ can be computed in polynomial time.
The $\graph$ has an independent set $I$ of size $M$ if and only if $\balls$ has a maximum intersection of size $\abs{\vertices} \abs{\edges} + M$ (see appendix).
Thus, $\varphi$ is a polynomial-time reduction from $\maxis$ to $\maxint$, and Theorem~\ref{thm:max-int} follows.

Given our proof of Theorem~\ref{thm:max-int}, Theorem~\ref{thm:hardness} follows by constructing instances of $\atkknn$ realizing each instance of $\maxint$ constructed in the proof of Theorem~\ref{thm:max-int}. To this end, we note that choosing
\begin{talign}
\evalpool &= \set{(\radius \basis_i, y^-) \sucht i \in \vertices} \cup \set{(-\radius \basis_i - \radius \basis_j, y^-) \sucht ij \in \edges}\\
\trainingset &= \set{(2 \radius \basis_i, y^+) \sucht i \in \vertices} \cup \set{(- (2 \radius - \eps) (\basis_i + \basis_j), y^+) \sucht ij \in \edges}
\end{talign}
suffices (see Figure~\ref{fig:hardness}).
The proof appears in the appendix.

%% file: methods.tex
Having proven the hardness of \atkr's task, we now present two algorithms for attacking $k$NN. The first algorithm finds a one-point attack that maximizes the \emph{coverage} of the attack point---i.e., the number of influencing regions in which it is contained. The second attack greedily performs one-point attacks to construct a full attack. Since finding the optimal single point to add to \(\trset\) is NP-Hard, we compute a one-point attack with an anytime algorithm that finds successively better solutions over time.

Our algorithms use Algorithm~\ref{alg:construct_hyperspheres} as a preprocessing step to convert a training set and target pool to a labeled collection of balls.\footnote{A simple optimization is to remove from \(\labeledballs\) all balls with radius 0 and any for which \(y^{\texttt{orig}}_i = y\) when \(\atkrlabelset = \{y\}\). For ease of notation we ignore this optimization, maintaining \(|\mathbf \labeledballs| = |\evalpool|\).} The problem of mounting an optimal attack in which a single point is added to \(\atkrset\) with multiplicity \(\ell\) is completely determined by \(\labeledballs\) via Equation~(\ref{eqn:tsi_defn}).

\subsection{Constructing a One-Point Attack}

We present an anytime algorithm which \textbf Constructs a \textbf Hypergraph for  \textbf One-\textbf Point  \textbf Poisoning  \textbf Attacks (\chg).
Consider the hypergraph \(\hypergraph = (V, E)\) where each vertex corresponds to one of the influencing regions and the hyper-edge \(\mathbf e = \{v_i, \ldots, v_j\}\) exists if the intersection \(v_i \cap \cdots \cap v_j\) is non-empty. Each hyper-edge represents a set of target points which can be influenced %
by a single training point with some multiplicity, \(\ell\). We say that a point \(\x'\) \emph{covers} \(\mathbf{e}\) if \(\x'\) is contained in the intersection of influencing regions in \(\mathbf e\). We denote \(\coverage(\x)\) the size of the maximal hyper-edge covered by \(\x\). While \(\x\) lies in the intersection of \(\coverage(\x)\) influencing regions, adding \((\x, y^+)\) to \(\trset\) will only increase \(\tsi\) by \(\coverage(\x)\) if the point is added with sufficient multiplicity---i.e., the the maximum cost of any labeled ball in \(\mathbf e\). Given a maximum allowable multiplicity \(\ell\), \chg\ returns both an attack point \(\x\) that covers some hyper-edge \(\mathbf e\) and the minimum multiplicity \(\ell_\x \leq \ell\) needed to change the classification of points in \(\mathbf e\).

An optimal \emph{coverage} single-point attack places a new training point in intersection with the largest \(\coverage\). In the case of attacking $1$NN, an optimal coverage single point attack also maximizes the \(\tsi\) of any single-point attack. For \(k > 1\) adding any single point (with multiplicity \(1\)) may have \(\tsi = 0\), though adding the point with multiplicity \(\ell = k' = \lceil k / 2 \rceil\) sufficies to achieve \(\coverage(\x) = \tsi(\x, y)\). Thus, in general, we must have \(\ell \geq k'\) in order to guarantee that a nontrivial (single point) attack exists.

\input{CHG_pseudocode}

To compute $\hypergraph$, we utilize the following observation:
If an edge \(\mathbf e\) of size \(m\) is not an element of \(\edges\), then neither is any superset of \(\mathbf e\).
We use a dynamic-programming approach which first considers edges of size \(m = 1\), then \(m = 2\), etc.
The run time is highly dependent on the structure of \(\graph\).
When all balls are identical there are an exponential number of edges to check, whereas if no two balls intersect, the algorithm terminates before checking for any size 3 sets.
By keeping track of the best point (a point with maximal \tsi) as the algorithm considers larger cardinality collections of balls we allow for early termination given a time budget, making our solution deployable as an anytime algorithm.

\subsection{Constructing a Full Attack}

For the full attack, \atkr\ has a initial budget of \(\atkrbudget\) points to add to \(\atkrset\). We adopt the following greedy strategy: invoke \(\chg\) to find an optimal single-point attack with multiplicity \(\ell\) set to the minimum of \(\atkr\)'s remaining budget and \(k' = \lceil k / 2 \rceil\). Then add the point \(\x\) returned by \(\chg\) to \(\atkrset\) with multiplicity \(\ell_\x = \min \set{c(\x') \sucht \x \in \IR(\x', y, \trset \cup \atkrset)}\), and deduct \(\ell_\x\) from the remaining budget. As \(\chg\) is an anytime algorithm, we split a total time budget \(\atkrtimebudget\) evenly across the \(\atkrbudget\) calls.\footnote{Examining more sophisticated scheduling is left as future work.} We call this method \textbf Greedy \textbf Identification of a \textbf Training-\textbf Time \textbf Attack via \chg~(\greedyatk). It is a greedy anytime algorithm for the general \(\atkrbudget > 1\) attack.
The pseudocode is shown in Algorithm~\ref{alg:full}.
\input{full_atk_pseudocode}

Since \chg\ is an anytime algorithm, it may not always return an optimal single-point attack. Nonetheless, the following theorem asserts that if each call to \chg\ runs to completion, and thus returns an optimal single-point attack, then the $\atkrbudget$-point attack found by \greedyatk\ is approximately optimal.
In the appendix, we prove and generalize this result, and describe further practical optimizations.

\begin{theorem}
\label{thm:greedy_bound}
Let \(k' = \lceil k / 2 \rceil\). Suppose each call to \chg\ in \greedyatk\ returns an optimal single-point attack against \(k\)NN. Then the score of the $\atkrset$ returned by \greedyatk\ is a $\frac{1}{k'}(1 - 1 / e)$-fraction of the optimal $\atkrbudget$-point attack's score.
\end{theorem}

  In practice, a user can detect when \chg\ has run to completion (hence yielding an optimal result). Theorem~\ref{thm:greedy_bound} can then be used as a post-hoc bound. If each call to \chg\ returns a solution that within a $\beta$ factor of the optimal single-point attack, then the solution returned by \greedyatk\ is a \(\frac{1}{k'}(1 - 1 / e^\beta)\)-factor approximation to the optimal \(\atkrbudget\)-point attack. Details appear in the appendix.

We provide a full proof of Theorem~\ref{thm:greedy_bound} in the appendix, and provide a high level overview of the argument here. First, consider the case \(k = 1\). The problem of mounting an optimal \(\atkrbudget\)-point attack can then be reduced to the \emph{maximum coverage problem} (MCP) defined by~\cite{Hochbaum1998analysis}. An instance of the \(\atkrbudget\)-MCP consists of (1) a universal set \(U\) of elements where each \(x \in U\) has an associated weight \(w(x)\), and (2) a family \(\mathcal{S}\) of subsets of \(U\). The goal is to find a family of sets \(\mathcal{F} \subseteq \mathcal{S}\) of size \(\atkrbudget\) that maximizes the quantity \(\sum_x w(x)\), where the sum is taken over \(\bigcup \mathcal{F}\).

For \(\atkknn\), we have \(U = \evalpool\), and weights \(w(\x', y')\) are \(\pm 1\) depending on if \(y' = y^+\). Each \(S \in \mathcal{S}\) consists of elements in \(\evalpool\) whose influencing regions mutually intersect.
\ifsqueeze \color{black}
\color{black} \else \color{black}
More specifically, \(S \in \mathcal{S}\) if there exists a point \(\x\) such that
\[
S = \set{(\x', y') \in \evalpool \sucht \x \in \IR(\x', y, \trset)}.
\]
\color{black} \fi

In general, MCP is NP-hard, but \citet{Hochbaum1998analysis} show that choosing sets greedily from \(\mathcal{S}\) gives a \((1 - 1 / e)\)-factor approximation to the optimal solution. Under the assumptions that \chg\ returns optimal single-point attacks and \(k = 1\), \greedyatk\ simulates Hochbaum and Pathria's algorithm, whence Theorem~\ref{thm:greedy_bound} follows.

For \(k > 1\), an optimal \(\atkrbudget\)-point attack is no longer equivalent to MCP. Indeed, a point \((\x', y') \in \evalpool\) may need to be ``covered'' with multiplicity greater than \(1\) in order to have its prediction flipped, and it may be covered by multiple different sets in \(\mathcal{S}\). Thus, the \(k > 1\) case appears to be strictly more general than the MCP and the related ``budgeted MCP'' studied in~\cite{Khuller1999-budgeted}.

To obtain Theorem~\ref{thm:greedy_bound} for \(k > 1\), let \(\atkrset_\opt\) be an optimal \(\atkrbudget\)-point attack---i.e., one that maximizes \(\tsi(\atkrset)\). For any set \(\atkrset\), let
\begin{align}
\nonumber \coverage(\atkrset) = \abs{\set{(\x', y') \in \evalpool \sucht \atkrset \cap \IR(\x', y, \trset) \neq \varnothing}}.
\end{align}
When \(k = 1\), we have \(\tsi(\atkrset) = \coverage(\atkrset)\), and for all \(k\), \(\tsi(\atkrset) \leq \coverage(\atkrset)\).
Let \(\atkrset_\cov\) be a \(\atkrbudget'\)-point attack that maximizes \(\coverage(\atkrset)\). Then we have
\begin{equation}
  \label{eqn:score-coverage}
\nonumber  \tsi(\atkrset_\opt) \leq \coverage(\atkrset_\opt) \leq k' \coverage(\atkrset_\cov).  
\end{equation}
Now consider the set \(\atkrset^*\) returned by \greedyatk\ in an execution in which each call to \chg\ returns an optimal point. Since points are added to \(\atkrset^*\) with multiplicities as large as \(k'\), it may contain as few as \(\atkrbudget' = \lfloor \atkrbudget / k' \rfloor\) distinct points, and we have
\[
\coverage(\atkrset_\cov) \leq (1 - 1/e)^{-1} \coverage(\atkrset^*).
\]
By construction---since each point \(\atkrset^*\) is taken with sufficient multiplicity---we additionally have \(\tsi(\atkrset^*) = \coverage(\atkrset^*)\).
Combining the previous two expressions with~(\ref{eqn:score-coverage}) proves Theorem~\ref{thm:greedy_bound}.

%% file: CHG_pseudocode.tex
 \begin{algorithm}[t]
\small
   \DontPrintSemicolon
   \caption{\chg}
   \setcounter{AlgoLine}{0}
   \label{alg:chg}
   \SetKwFunction{FhasOverlap}{hasOverlap}
   \SetKwProg{Fn}{Function}{:}{}
   \KwInput{\\
     \(\trset = \{(\x^{\text{tr}}, y^{\text{tr}})\}_i\),  \(\evalpool = \{(\x^{\text{ev}}, y^{\text{ev}})\}_i\), \(y^+\)\\
     \(\atkrtimebudget, \ell\) \hfill Time budget,  Point budget\\
   }
   
   \KwOutput{\\
     \((\x, \ell_{\x})\)  \hfill Attack point and associated cost}

   \(\mathbf{H} \gets \texttt{ConstructIRs}(\trset, \evalpool, y^+)\)
   
   \(\texttt{edges} \gets [\emptyset]\)
   
   \(\texttt{done} \gets \texttt{False}\)
   
   \(\texttt{best} \gets \bot, \ell_{\texttt{best}} \gets 0\)
   
   \While{not \texttt{done} \(\wedge\) \(\atkrtimebudget\) not exhausted}
         {
           \(\texttt{newEdges} \gets \emptyset\)
           
           \For{\(\potentialedge \in \{e \cup \{h\} ~|~ e \in \texttt{edges}, h \in \mathbf H\)\}}
                     {%
                       \If {all \((|\potentialedge|-1)\)-sized subsets of \(\potentialedge\) are in \texttt{edges}\label{line:necessarycondition}\\                        
                         \(\wedge ~\text{There exists a point of mutual overlap}\)\label{line:shortcircuit}
                       }
                           {
                             \(\texttt{newEdges} \gets \texttt{newEdges} \cup \{\potentialedge\}\)
                             
                             \(\mathbf x \gets \)\text{~any point in intersection}
                             
                             \If{\(\tsi(\x, y^+, \ell, \mathbf H) > \tsi(\texttt{best}, y^+, \ell, \mathbf H)\)} 
                                {
                                  \(\texttt{best} \gets \x\)
                                  
                                  \(\ell_{\texttt{best}} \gets \max\set{c \sucht (B, v, c) \in \potentialedge}\)
                                }
                           }
                     }
                     \(\texttt{edges} \gets \texttt{newEdges}\)
                     
               \If{\(\texttt{edges} = \emptyset\)}
                  {\(\texttt{done} \gets \text{True}\)}
                  
         }
         \Return \((\texttt{best}, \ell_{\texttt{best}})\)
         
 \end{algorithm}

%% file: full_atk_pseudocode.tex
 \begin{algorithm}[t]
\small
   \DontPrintSemicolon
   \caption{\greedyatk}
   \setcounter{AlgoLine}{0}
   \label{alg:full}
   \KwInput{\\
     \(\trset = \{(x^{\text{tr}}, y^{\text{tr}})\}_i\),  \(\evalpool = \{(x^{\text{ev}}, y^{\text{ev}})\}_i\), \(y^+\)
     
     \(\atkrtimebudget, \atkrbudget\) \hfill Time budget, Point budget\\
   }
   \KwOutput{\\
     \(\atkrset = \{((\x_1, y_1), \ell_1), \ldots, ((\x_{\atkrbudget}, y_m), \ell_m)\}\) \hfill The attack
   }
   \(\atkrset \gets [~], \atkrbudget_{\text{rem}} \gets \atkrbudget\)
      
   \While{\(\atkrbudget_{\text{rem}} > 0\)}
       {
         \(((\x, y), \ell)  \gets \chg\left(\trset \cup \atkrset, \evalpool, y^+, \frac \atkrtimebudget {\atkrbudget}, \atkrbudget_{\text{rem}}\right)\)
         
         if \(\perp\) returned, then \textbf{break}
         
         \(\atkrset \gets \atkrset \cup \{(\mathbf x, y^+)\}\)
         
         \(\atkrbudget_{\text{rem}} \gets \atkrbudget_{\text{rem}} - \ell\)
       }
       \Return \(\atkrset\)

 \end{algorithm}

%% file: exp.tex
\input{synthetic_table}

Having laid the foundation for a theoretical understanding of attacking \(k\)NN at training time,
we turn to the question of how well our attack works in practice with our ultimate application being defense.
We explore the effectiveness of \chg\ on synthetic data in scenarios small enough to compute the optimal one-point attack, then turn to real-world data and investigate \greedyatk\ on larger instances.
Driven by the results on the synthetic data and with defense in mind, we look at the effect of dimensionality reduction on the attacks. 
Experiments were run on 128-core machines from AWS EC2.
The QCLPs were solved using Mosek (\citet{mosek}).
Pre-processing, plotting, and data analysis were performed with pandas
(\citet{reback2020pandas}), scikit-learn (\citet{scikit-learn}) and matplotlib (\citet{matplotlib}).

In all experiments presented here we focus mostly on \(k\)NN using Euclidean distance,
deferring some experiments using the $\ell_\infty$-norm to the appendix. Given a set of
balls \(\potentialedge\), we determine whether or not there exists a point in
the intersection \(\bigcap_i \ball_i\). For the $\ell_\infty$-norm, a simple interval
intersection algorithm suffices. For the Euclidean norm, this is done by
constructing the following quadratically constrained linear program
(QCLP):\footnote{If \atkr\ is constrained in her selection of \(\mathbf x\),
  these constraints can be added to~(\ref{eqn:qclp}).}
\begin{equation}
  \label{eqn:qclp}
  \scalebox{0.9}{%
    $\min_\x\ 1 \text{ s.t. }
    \|(\x - \mathbf c_1)\|_2^2 \leq r_1^2, \ldots, \|(\x - \mathbf c_{|\potentialedge|})\|_2^2 \leq r_{|\potentialedge|}^2.$}
\end{equation}
The objective function of~(\ref{eqn:qclp}) is not important; if there exists any feasible point, then the balls have a point of mutual overlap.
To reduce the number of times a QCLP solver is invoked, we make a further optimization.
When considering a potential hyper-edge \(\potentialedge = \potentialedgeelem_1, \ldots \potentialedgeelem_{|\potentialedge|}\) the following condition is necessary for \(\potentialedge\) to be in \(E\):
All subsets of \(\potentialedge\) of size \(|\potentialedge|-1\) must be in \(E\).
This check is computationally efficient because, when the check is performed, all edges of size \(|\potentialedge| - 1\) have already been comptued.
This is implemented on line~\ref{line:necessarycondition} %
in Algorithm~\ref{alg:chg}. If the first clause is false, the latter is not checked (the QCLP solver is not invoked).

\subsection{Synthetic Experiments}
\label{synthexp}

We construct two families of synthetic experiments---\texttt{Uniform} and
\texttt{Normal}---small enough to compute the
optimal one-point attack. To create each instance,
we sample $m$ training points and 10 target points from the appropriate distribution, $\mathbf{X} \sim
\mathrm{Unif}([0,1]^d)$ for \texttt{Uniform}, $\mathbf{X} \sim
\mathrm{Normal}(\mathbf{0}, \mathbf{I}_d)$ for \texttt{Normal}.
For each family this procedure is repeated for the grid of values $m \in \{8, 16, 32, 64, 128\}$
and $d \in \{2, 4, 8, 16, 32\}$.
For each $(m,d)$ pair and family we perform 10 trials.
Averages are reported in Table~\ref{tab:synth}, with full details in Appendix \ifsqueeze \color{black} D \color{black} \else \color{blue} B \color{black} \fi.
Universally, \atkr~ is more successful in high dimenions.
This suggests a natural defense strategy: perform dimension reduction prior to learning.

\subsection{MNIST and HAPT}

We consider two real-world datasets.
MNIST (\citet{lecun2010mnist}) consists of \(28 \times 28\) greyscale images of handwritten digits \((d = 784)\).
Human Activity Recognition (HAPT) (\citet{anguita2013public}), is a
dataset of sensor recordings of subjects performing a range of daily activities \((d=561)\).
We use 6 of the labels from HAPT, omitting labels representing transitions
between activities.

For MNIST, we consider 10 attackers, one targeting each digit, and for HAPT we consider 6.
Each attacker has an target pool of 50 points sampled uniformly at random from the original set that are originally classified as their label of interest.
To evaluate dimensionality reduction as a form of defense we used Principle Component Analysis (PCA) (\citet{pearson1901lines}) on each dataset to reduce its dimension to \(d = 2^1, 2^3, 2^5, 2^7, 2^9\).
We train a \(1\)NN classifier using 10,000 and
6,000 points from MNIST and HAPT, respectively, (1,000 from each class).
Every attacker uses \greedyatk~ with a budget of \(\atkrbudget = 1,5, 10, 20\) and a time budget of \(\atkrbudget\) minutes.
\atkr's effectiveness as a function of which class is attacked are omitted for brevity, as we noticed no definitive connection between classes that are difficult to classify and classes that are difficult to attack.
Results are presented in Figure~\ref{fig:real_world_results}.

\input{real_world_plot}

The attackers are very effective.
Even with a modest budget of \(b = 10\), \atkr~ is able to affect the predicted label of over half the target pool in both original datasets.
Secondly, as in the synthetic experiments, the effectiveness of the attacks decreases as the data's dimension is reduced.

For dimensionality reduction to be a useful defense strategy, the defender would need to reduce the dimension while maintaining good performance on the underlying task. To understand the effects of reducing the dimension on the learning process, we report the percentage of variance explained (a measure independent of the choice of learner) and the zero-one loss of \(1\)NN measured on a hold-out set of size 1,000 (randomly selected, disjoint from the training set) in Table~\ref{tab:real_world_results}.
\(k\)NN continues to perform well even as the dimension is reduced (a well understood phenomenon) until the dimension is very low.
\atkr's score, on the other hand, steadily decreases with the dimension.
In the case of MNIST data, reducing to \(32\) dimensions decreased the attacker's effectiveness by between 26\%  (\(\atkrbudget = 1\)) and 12\% (\(\atkrbudget = 20\)) while increasing prediction error by less than 1\%.
PCA served as a more costly defense for the HAPT data.
At \(d = 32\) the attacker's score decreased by between 11\% and 8\%, but the prediction loss increased by about 28\%.
 Dimensionality reduction can be an effective defense for \(k\)NN against training-time attacks. %
Experiments on other real-world datasets showed qualitatively similar behavior and results are presented in Appendix D.

\input{real_world_results_table}

%% file: synthetic_table.tex
\begin{table}[t]
  \small
  \centering
\begin{tabular}{|c|c|c|c|c|c|c|}
\hline
\(d\):           & & 2    & 4    & 8    & 16   & 32    \\ \hline
\texttt{Normal}  & $\ell_2$ & 2.72 & 3.68 & 5.32 & 8.34 & 10.00 \\ \hline
                 & $\ell_\infty$ & 2.70 & 3.72 & 5.10 & 7.38 & 8.86 \\ \hline
\texttt{Uniform} & $\ell_2$ & 2.70 & 3.42 & 5.44 & 8.36 & 10.00 \\ \hline
                 & $\ell_\infty$ & 2.70 & 3.56 & 6.10 & 9.40 & 10.00 \\ \hline
\end{tabular}
\caption{{\bf \atkr's Scores on Synthetic Data.} For dimensions $d = \{2, 4,
  8, 16, 32\}$ and both $\ell_2$ and $\ell_\infty$ norms, we report the \atkr~score (out of 10), averaged over trials and training set sizes. \atkr's effectiveness increases with dimension. All standard error of the means were less than 8.2\%.
 \ifsqueeze
 Details are presented in Appendix D.
 \else
 More details are presented in the appendix.
 \fi
}
\label{tab:synth}
\end{table}

%% file: real_world_plot.tex
\begin{figure}[t]

\centering
\includegraphics[width=\defaultfigsize]{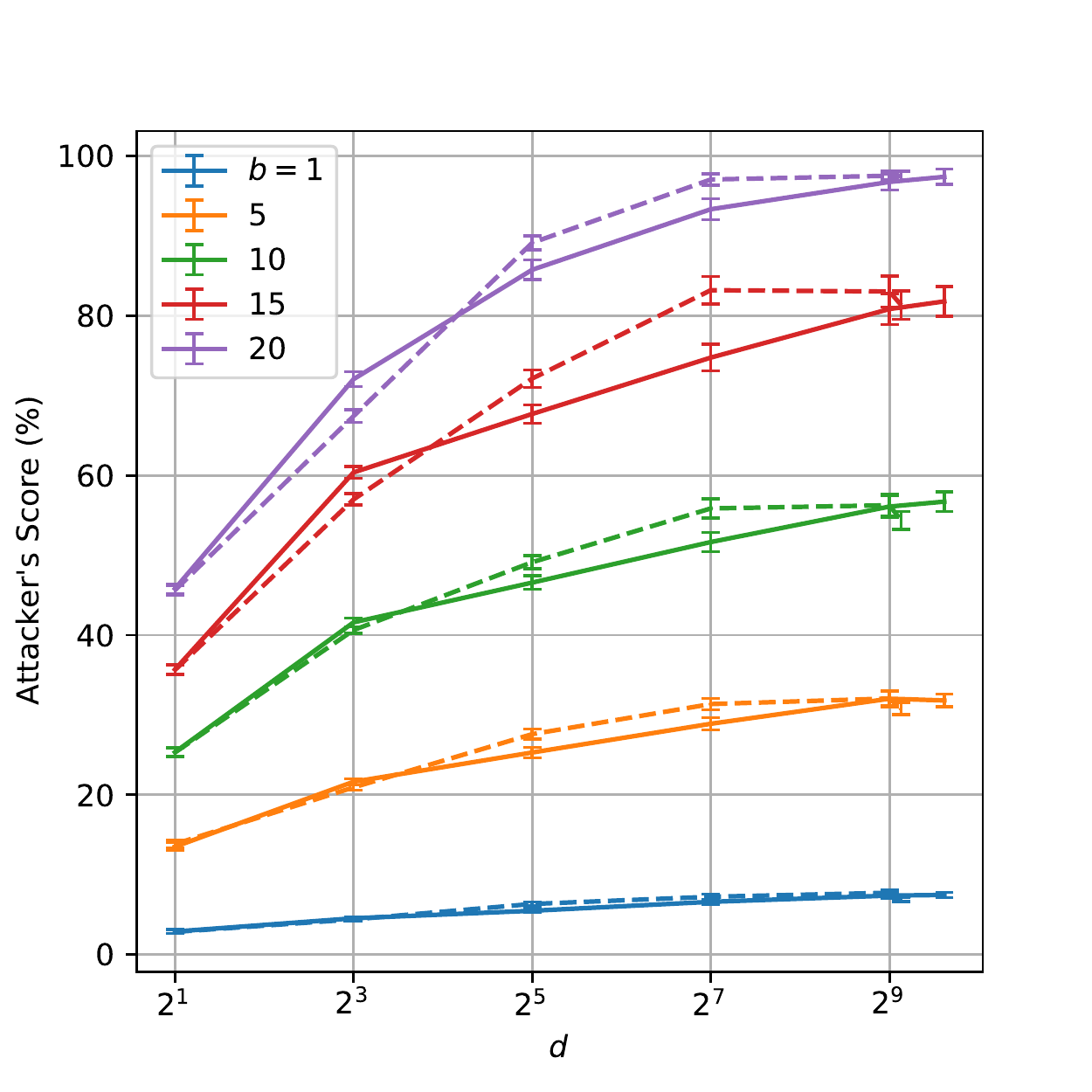}
\caption{{\bf ATKR's Effectiveness as a Function of Dimension.}
  For the original datasets and those reduced via PCA, we plot \atkr's score as a percentage of \(|\evalpool|\) averaged over 10 trials for each label\ifsqueeze.\else, for various attacker budgets.\fi~
  Error bars denote standard error of the mean.
  Attacks on MNIST (HAPT) datasets are shown as solid (dashed) lines.
  Universally, reducing the dimension is an effective defense.
}
\label{fig:real_world_results}
\end{figure}

%% file: real_world_results_table.tex
\begin{table}[t]
  \centering
  \small
\begin{tabular}{|c|c|c|c|c|}\hline
                                & \multicolumn{2}{c|}{MNIST} & \multicolumn{2}{c|}{Hapt} \\ \hline
\multicolumn{1}{|c|}{\(d\)}     & \texttt{HOL}               & \texttt{VarE}             & \texttt{HOL}         & \texttt{VarE}            \\ \hline\hline
\multicolumn{1}{|c|}{Original}  & 0.102 (0.483)              & 1.000                     & 0.087 (0.285)        & 1.000          \\ \hline 
\multicolumn{1}{|c|}{512}       & 0.126 (0.119)              & 0.999                     & 0.093 (0.141)        & 1.000         \\ \hline
\multicolumn{1}{|c|}{128}       & 0.111 (0.145)              & 0.936                     & 0.081 (0.158)        & 0.983     \\ \hline
\multicolumn{1}{|c|}{32}        & 0.103 (0.143)              & 0.745                     & 0.111 (0.154)        & 0.891     \\ \hline
\multicolumn{1}{|c|}{8}         & 0.201 (0.203)              & 0.438                     & 0.217 (0.158)        & 0.774      \\ \hline
\multicolumn{1}{|c|}{2}         & 0.613 (0.619)              & 0.169                     & 0.431 (0.426)        & 0.661     \\ \hline
\end{tabular}
\caption{{\bf Loss and Variance Explained.} We report the zero-one loss on a
  held out dataset (\texttt{HOL}) for $\ell_2$-norm ($\ell_\infty$-norm) and the proportion of variance explained (\texttt{VarE}) for the each dataset (the original and those with dimension reduced by PCA).
  Reducing the dimension to 32 or 128 hardens the learner against attack while maintaining similar \texttt{HOL}.
}
\label{tab:real_world_results}
\end{table}

%% file: relwork.tex
Since its introduction~(\citet{fix1951discriminatory}), \(k\)NN and related algorithms (e.g., \(r\)NN) have been extensively studied in contexts where no adversary is at play (\citet{shakhnarovich2006nearest,bhatia2010survey}).

Training-time (a.k.a.~poisoning or input manipulation) attacks were developed against various machine learning algorithms from deep networks~(\citet{koh2017icml}) to support vector machines~(\citet{biggio2012icml}) to linear regression models~(\citet{jagielski2018oakland,alfeld2019attacking}) to collaborative filtering~(\citet{li2016nips,chen2020data}) and online centroid anomaly detection~(\citet{kloft2010aistats}).
We complement this by furthering knowledge of another canonical learning method.
We believe we are the first to investigate (theoretically or otherwise) training-time attacks against \(k\)NN.

Deployment-time attacks, often called ``adversarial examples'' or ``adversarial perturbations'', have been given a great deal of attention, especially against neural networks in recent years.
With a focus on graph neural-networks Entezari {\it et al.} similarly observed that attackers are generally are less effective in low dimensions (\citet{entezari2020all}).
We direct the reader to~\citet{biggio2018wild,vorobeychik2018adversarial,joseph2018adversarial} for a general review and to \citet{serban2020adversarial,hazan2017adversarial} specifically for neural networks.
Others have developed attacks against and analyzed the robustness of \(k\)NN and other non-parametric methods~\citet{wang2018analyzing,bhattacharjee2020non}).
In contrast to our work, they model an attacker acting at deployment-time, not training-time.

\input{relwork-hardness}

%% file: relwork-hardness.tex
$\atkknn$ is closely related to the ``product positioning problem'' in marketing theory~\cite{Albers1977procedure}, whose complexity is analyized in~\cite{Crama1995complexity}.
Their analysis gives an alternative proof of the special case of our Theorem~\ref{thm:max-int} when \(p = 2\) via a reduction from the maximum clique problem.
Our proof of Theorem~\ref{thm:max-int} is an adaptation of  the argument in~\cite{Amaldi1995complexity}, which shows that finding the maximum number of feasible linear constraints is NP-hard.
The advantage of our proof of Theorem~\ref{thm:max-int} is threefold: it is more general in that it covers \(\ell_p\) norms for \(p \in (1, \infty)\), it is both technically simpler than the argument of~\citet{Crama1995complexity} and it is conceptually consistent with the remainder of our paper.
\citet{Crama1995complexity}~also describes an algorithm for $\maxint$ that runs in time $O(n^{\dimn + 1})$, where $n$ and $\dimn$ denote the number of points (i.e., size of our training set) and the dimension, respectively.
The high-degree polynomial run-time of Crama et al.'s algorithm, however, is inherent to their approach, as they first test a set of $\Theta(n^\dimn)$ potential single-point attacks before evaluating the scores of these attacks.
As an anytime algorithm, \chg\ finds good (albeit not necessarily optimal) solutions quickly.

%% file: conc.tex
We investigated the task of performing training-time attacks against nearest-neighbor based classifiers.
We model an attacker adding points to the training set with the goal of affecting the classifications of points in an target pool.
We proved that finding an optimal attack against \(k\)NN is NP-Hard, even when the attacker can add only a single data point to the training set.
We introduced the anytime algorithm \chg\ to efficiently compute an effective one-point attack and \greedyatk\ to compute a general attack.
We provided a bound for \greedyatk\ when \chg\ finds the optimal solution, and this bound holds in general if \chg\ is replaced with a constant-approximation algorithm.
\ifsqueeze \color{black}
With an eye toward defense strategies, we conducted an empirical investigation on synthetic and real-world data.
\color{black} \else \color{black}
With an eye toward defense strategies, we conducted an empirical investigation of the effects of dimensionality reduction on synthetic and real-world data.
\color{black} \fi
Our experiments demonstrated both that the attack is highly effective in practice and that dimensionality reduction is often an effective defense strategy.

We close with open problems as potential future work:
1.~Is there a polynomial-time, constant-factor approximation algorithm which solves the same problem as \chg?
Such an algorithm with approximation ratio \(\alpha\) would yield a \((1 - e^{-\alpha}) / k'\) approximation when used as a subroutine by \greedyatk.
2.~Is there a bound for \greedyatk's performance that is independent of \(k\)?
3.~Our defense reduces the dimensionality \(d\).
We observe empirically that this reduces the effectiveness of the optimal attack.
However, the worst-case runtime of \chg\ decreases with \(d\).
Are there effective defenses which act by increasing the computational burden of the attacker instead of (or in addition to) reducting the effectiveness of the optimal attack?

%% file: appendix.tex
\begin{onecolumn}

  \begin{center}
    \huge{Appendices for\\
      Training-Time Attacks Against k-Nearest Neighbors}
  \end{center}
  
\ifsqueeze
  \section*{Appendix A: Notation Table}

\input{app_notation}
\vfill  \pagebreak
  
\section*{Appendix B: Additional Proofs}
\input{app_proofs}

\vfill \pagebreak

\section*{Appendix D: Additional Experiment Results}

\input{app_exps}
\else

\section*{Appendix A: Additional Proofs}
\input{app_proofs}
\vfill \pagebreak

\section*{Appendix B: Additional Experiment Results}
\input{app_exps}
\fi

\end{onecolumn}

%% file: app_notation.tex
\input{notation_table}

%% file: app_proofs.tex
\subsection*{Proofs for Theorems~\ref{thm:hardness}  and~\ref{thm:max-int} from Section~\ref{sec:hardness}}

We use the following three lemmas to establish Theorem~\ref{thm:max-int}.

\begin{lemma}
  \label{lem:ball-coordinate-sign}
  Suppose $\x \in \Rd$. Then for each $i \in \vertices$ and $ij \in \edges$ we have (1) $\x \in \ball_i \implies x_i \geq 0$, and (2) $\x \in \ball_{ij} \implies x_i + x_j < 0$. Therefore, for every edge $ij \in \edges$, we have $\ball_i \cap \ball_j \cap \ball_{ij} = \varnothing$.
\end{lemma}
\begin{proof}%
  For the first claim, observe that if $x_i < 0$, then the first term in the sum of~(\ref{eqn:balli}) satisfies $\abs{x_i - \radius}^p > \radius^p$, hence the inequality in~(\ref{eqn:balli}) is not satisfied.

  For the second claim, suppose $x_i + x_j \geq 0$. Without loss of generality, assume that $x_i \geq \abs{x_j}$. Suppose $x_j < 0$ (the case where $x_i, x_j \geq 0$ is trivial). Then we compute:
  \begin{align*}
    \abs{\radius}^p &= \abs{\frac 1 2 (\radius+ x_j) + \frac 1 2 (\radius - x)}^p\\
    &\leq \paren{\frac 1 2 \abs{\radius+ x_j} + \frac 1 2\abs{\radius- x_j}}^p\\
    &\leq \frac 1 2 \abs{\radius+ x_j}^p + \frac 1 2 \abs{\radius- x_j}^p\\
    &\leq \frac 1 2 \abs{\radius+ x_j}^p + \frac 1 2 \abs{\radius+ x_i}^p.
  \end{align*}
  The first inequality holds from the triangle inequality combined with the monotonicity of $f(x) = x^p$; the second inequality is Jensen's inequality, and the final equality comes from the assumption that $x_i \geq \abs{x_j}$. Combining the first and final expression, we obtain
  \[
  \abs{\radius+ x_i}^p + \abs{\radius+ x_j}^p \geq 2 r^p
  \]  
  which implies that the inequality in~(\ref{eqn:ballij}) is not satisfied.

  Finally, suppose $\x \in \ball_i \cap \ball_j \cap \ball_{ij}$ Then by item~1 in the lemma we have $x_i, x_j \geq 0$, while by item~2, we have $x_i + x_j < 0$, a contradiction.
\end{proof}

While Lemma~\ref{lem:ball-coordinate-sign} shows that any three balls of the form $\ball_i, \ball_j, \ball_{ij}$ are not mutually intersecting, every two of the balls do intersect. In fact, the following lemma shows that every sub-family $\balls'  \subseteq \balls$ is mutually intersecting, so long as $\balls'$ does not contain three balls of the form $\ball_i, \ball_j, \ball_{ij}$---that is, $\balls'$ is mutually intersecting so long as the set of indices $i \in \vertices$ with $\ball_i \in \balls'$ forms an independent set in $\graph$.

\begin{lemma}
  \label{lem:indep-set}
  Let $\graph$ be a graph on $\dimn$ vertices. Suppose $\balls = \varphi(\graph)$ and $I \subseteq \vertices$ forms an independent set in $\graph$. Let $\balls_I = \balls \setminus \set{\ball_j \sucht j \notin I}$. Then choosing $\eps = \frac 1 \dimn$ and $r = 9$, we have $\bigcap_{\ball \in \balls_I} \ball \neq \varnothing$.
\end{lemma}
\begin{proof}
  For $i \in [\dimn]$ and $I \subseteq \vertices$ an independent set, define the vector $v^I \in \R^\dimn$ as follows:
  \begin{equation}
    v^I_j = 
    \begin{cases}
      \frac 1 2 &\text{if } i \in I\\
      -1 &\text{otherwise.}
    \end{cases}
  \end{equation}
  We will show that for a suitable choice of $\eps$ and sufficienty small value of $t > 0$, a point of the form $\x = t v^I$ is contained in $\bigcap \balls_I$. That is, we must show that
  \begin{align}
    \norm{\x - \basis_i}_p^p &\leq 1, and \label{eqn:vertex-ball}\\
    \norm{\x + \basis_i + \basis_j}_p^p &\leq 2 (r - \eps)^p = 2 r - 2 p \eps - O(\eps^2). \label{eqn:edge-ball}
  \end{align}
  To this end, for each $i \in I$, we first observe that
  \begin{align*}
    \norm{\x - \basis_i}_p^p = \abs{1 - x_i}^p + \sum_{j \neq i} x_j^p = 1 - \frac 1 2 p t + d t^p + O(t^2).
  \end{align*}
  From this expression, Taylor's theorem implies that for all sufficiently small $t > 0$, Equation~(\ref{eqn:vertex-ball}) is satisfied.
  
  Similarly, for each $ij \in \edges$ with $j \notin I$, we compute
  \begin{align*}
    \norm{\x + \basis_i + \basis_j}_p^p &= (1 - t)^p + (1 - x_i)^p + \sum_{\ell \neq i, j} x_\ell^p\\
    &= \abs{1 + x_i}^p + \abs{1 + x_j}^p + \sum_{\ell \neq i, j} x_\ell^p\\
    &\leq \abs{1 + \frac 1 2 t}^p + \abs{1 - t}^p + \sum_{\ell \neq i, j} t^p\\
    &= 2 - \frac 1 2 p t + d t^p + O(t^2).
  \end{align*}
  Setting $\eps = t^2$ and again applying Taylor's theorem, we obtain that for sufficiently small $t$, Equation~(\ref{eqn:edge-ball}) is also satisfied. Thus, for sufficiently small $t$, we have $\x \in \bigcap \balls_I$, as desired.
\end{proof}

Next, we establish the converse of Lemma~\ref{lem:indep-set}: if $\pt$ is contained in many balls $\ball \in \balls$, then $\graph$ contains a large independent set.

\begin{lemma}
  \label{lem:overlap-is}
  Let $\graph = (\vertices, \edges)$ be a graph with $\dimn = \abs{\vertices}$ vertices and $m = \abs{\vertices}$ edges, and let $\balls = \varphi(\graph)$. Suppose there exists a point $\pt \in \Rd$ contained in the intersection of $m \dimn + M$  balls $\ball \in \balls$ for some $M \geq 1$ (where each ball of the form $\ball_{ij} \in \balls$ is taken with multiplicity $\dimn$). Then $\graph$ contains an independent set $I \subseteq \vertices$ of size $M$.
\end{lemma}
\begin{proof}%
  Suppose $\pt = (p_1, p_2, \ldots, p_\dimn)$ is contained in $m \dimn + M$ balls $\ball \in \balls$. First, we claim that $\pt \in \bigcap_{ij \in \edges} \ball_{ij}$. Since each $\ball_{ij}$ appears with multiplicity $\dimn$ and $\abs{\balls} = m \dimn + \dimn$, $\pt \notin \ball_{ij}$ implies that $\pt$ is contained in at most $\abs{\balls} - \dimn = m \dimn$ balls, contradicting the hypothesis. Therefore, we have $\pt \in \bigcap_{ij \in \edges} \ball_{ij}$, as claimed.

  Now consider the set $I = \set{i \sucht \pt \in \ball_i} \subseteq \vertices$. Note that $\abs{I} = M$ by the hypothesis. We claim that $I$ is an independent set in $\graph$. To see this, suppose $i \in I$ so that $i \in \ball_i$, and suppose $ij \in \edges$. Since $i \in \ball_i \cap \ball_{ij}$ (by the previous claim), Lemma~\ref{lem:ball-coordinate-sign} implies that $\pt \notin \ball_j$. Therefore, $j \notin I$. Thus, for all $i \in I$ and $ij \in \edges$, we have $j \notin I$---i.e., $I$ is an independent set of size $M$.
\end{proof}

\begin{proof}[Proof of Theorem~\ref{thm:max-int}]
  Consider the reduction $\varphi$ described above. Clearly, for any graph $\graph = (\vertices, \edges)$, $\balls = \varphi(\graph)$ can be computed in polynomial time. By Lemmas~\ref{lem:indep-set} and~\ref{lem:overlap-is}, $\graph$ has an independent set $I$ of size $M$ if and only if $\balls$ has a maximum intersection of size $\abs{\vertices} \abs{\edges} + M$. Therefore, $\varphi$ is a polynomial-time reduction from $\maxis$ to $\maxint$, hence $\maxint$ is NP-hard.
\end{proof}

\begin{proof}[Proof of Theorem~\ref{thm:hardness}]
  Let $\graph = (\vertices, \edges)$ be a graph, and $\balls = \varphi(\graph)$ the corresponding $\maxint$ instance described above. We will show that $\balls$ is realized by an $\atkknn$ instance that can be efficiently constructed from $\graph$. Thus, the reduction from $\maxis$ to $\maxint$ described in Section~\ref{sec:hardness} yields a reduction from $\maxis$ to $\atkknn$, whence Theorem~\ref{thm:hardness} follows.

  Given $\graph$, take the evaluation pool $\evalpool$ to consist of the centers of all the balls $\ball \in \balls$. That is,
  \[
  \evalpool = \set{(\radius \basis_i, y^-) \sucht i \in \vertices} \cup \set{(-\radius \basis_i - \radius \basis_j, y^-) \sucht ij \in \edges}.
  \]
  Each $-\radius \basis_i - \radius \basis_j$ occurs with multiplicity $\dimn$. The training set $\trainingset$ is defined to be
  \[
  \begin{split}
    \trainingset =& \set{(2 \radius \basis_i, y^+) \sucht i \in \vertices}\\
    \quad&\cup \set{(- (2 \radius - \eps) (\basis_i + \basis_j), y^+) \sucht ij \in \edges}.
  \end{split}
  \]
  See Figure~\ref{fig:hardness} for an illustration of $\evalpool$ and $\trainingset$. It is straightforward to verify that (1) for each $(\radius \basis_i, y^-) \in \evalpool$, $(2 \radius \basis, y^+) \in \trainingset$ is its nearest neighbor, which is at distance $\radius$, and (2) for each $(-\radius \basis_i - \radius \basis_j, y^-) \in \evalpool$, $(- (2 \radius - \eps) (\basis_i + \basis_j), y^+) \in \trainingset$ is its nearest neighbor which is at distance $\sqrt 2 (\radius - \eps)$. Therefore, the pair $(\trainingset, \evalpool)$ realizes the $\maxint$ instance $\varphi(\graph)$. Since $(\trainingset, \evalpool)$ can clearly be computed from $\graph$ in polynomial time, the NP-hardness of $\atkknn$ follows from Theorem~\ref{thm:max-int}.
\end{proof}

The proof above establishes the NP-hardness of $\atkknn$ only for $\atkrbudget = 1$ and $k = 1$. NP-hardness for any fixed constant $\atkrbudget > 1$ and $k = 1$ can be established as follows. Given any instance $(\trainingset, \evalpool, 1)$ of $\atkknn$ constructed as in the proof of Theorem~\ref{thm:hardness}, form the instance $(\trainingset', \evalpool', \atkrbudget)$ as follows. For $i = 0, 1, \ldots, \atkrbudget - 1$, let $(\trainingset_i, \evalpool_i)$ be the sets of points formed by translating $(\trainingset, \evalpool)$ by $3 i \radius \basis_1$. Then take
\[%
\trainingset' = \bigcup_{i = 0}^{\atkrbudget - 1} \trainingset_i \quad\text{and}\quad \evalpool' = \bigcup_{i = 0}^{\atkrbudget - 1} \evalpool_i.
\]
It is straightforward to show that the $\maxint$ instance corresponding to $(\trainingset', \evalpool')$ consists of $\atkrbudget$ translated and non-intersecting instances of the original $\maxint$ instance. Therefore, solving the budget-$\atkrbudget$ problem on $(\trainingset', \evalpool')$ is equivalent to solving $\atkrbudget$ (identical) instances of the budget-$1$ instance. Since the latter problem is NP-hard, the $\atkrbudget > 1$ is NP-hard as well.

Essentially the same argument also yields the NP-hardness of attacking $k$NN for any (odd) $k > 1$. In particular, we must modify $\trainingset$ to ensure that the influencing region of each point in $\evalpool$ is the same as the $k = 1$ case depicted in Figure~\ref{fig:hardness}. To this end, it suffices modify $\trainingset$ as follows. For each $(\x_i, y_i) \in \evalpool$ add $(k-1)/2$ copies of $(\x, y^+)$ and $(k-1)/2$ copies of $(\x, y^-)$ to $\trainingset$. With this modification of $\trainingset$, it is straightforward to verify that adding any single point $(\x', y')$ with $\x' \in \ball_i$ will result in $(\x_i, y_i)$ being classified as $y'$. That is, $\ball_i = \IR(\x_i, y_i)$. Thus, Theorem~\ref{thm:hardness} follows for any odd $k > 1$ as well.

\subsection{Proof of Theorem~\ref{thm:greedy_bound}}

Towards proving Theorem~\ref{thm:greedy_bound}, we prove the following more general result, from which Theorem~\ref{thm:greedy_bound} follows taking $\beta = 1$.

\begin{theorem}
\label{thm:greedy_bound_general}
Suppose all elements $(\x, y') \in \evalpool$ have the same label, say $y' = y$.
Let $k' = \lceil k / 2 \rceil$, and suppose each call to \chg\ in \greedyatk\ returns a single-point attack whose score is a $\beta$-fraction of the optimal coverage single point attack ($0 \leq \beta \leq 1$). Then the score of the $\atkrset$ returned by \greedyatk\ is a $\frac{1}{k'}(1 - e^{-\beta})$-fraction of the optimal $\atkrbudget$-point attack's score. In particular, if each call to \chg\ returns an optimal single point attack (i.e., $\beta = 1$), then $\atkrset$'s score is within $\frac{1}{k'}(1 - 1/e)$ of the optimal attack.
\end{theorem}

Throughout the section, we fix a label $y$ that is the target label of all points in $\evalpool$. We first consider the case $k = 1$. In this case, the problem of computing an optimal $\atkrbudget$-point attack can be reduced to the \emph{maximum $\atkrbudget$-coverage problem}: given a universal set $U$, a weight function $w : U \to \R$, and a family $\calS$ of subsets of $U$, find $\atkrbudget$ subsets $S_1, S_2, \ldots, S_\atkrbudget \in \calS$ that maximizes $\sum_{(\x, y') \in S} w(\x)$ where $S = S_1 \cup \cdots \cup S_\atkrbudget$.

In our case, $U = \evalpool$, while $\calS$ is the family of subsets of $U$ consisting of sets of elements whose influencing regions mutually intersect: 
\[
\textstyle
\calS = \set{S \subseteq \evalpool \sucht \bigcap_{(\x, y) \in S} \ir(\x) \neq \varnothing}.
\]
For a point $(\x, y') \in \evalpool$ and target label $y$, the weight function $w$ is determined by
\begin{equation}
  \label{eqn:weight}
  w(\x) =
  \begin{cases}
    1 & \text{if } y = y' \text{ and } \learner[\x; D] \neq y\\
    -1 &\text{if } y \neq y' \text{ and } \learner[\x; D] = y'\\
    0 &\text{otherwise.}
  \end{cases}
\end{equation}
Therefore, the optimal single-point attack consists of finding $\arg \max_{S \in \calS} \sum_{\x \in S} w(\x)$. More generally, the optimal $\atkrbudget$-point attack consists of finding
\[
\arg\max_{S_1, \ldots, S_\atkrbudget \in \calS} w(S_1 \cup \cdots \cup S_{\atkrbudget}),
\]
where for $S \subseteq \evalpool$, we define
\[
w(S) = \sum_{(\x, y') \in S} w(\x).
\]
That is, in the case where $\abs{\atkrlabelset} = 1$, the optimal $\atkrbudget$-point attack reduces to solving the maximum $\atkrbudget$-coverage problem on $\calS$.

\cite{Hochbaum1998analysis} analyze a greedy algorithm for the maximum $\atkrbudget$-coverage problem. Specifically, they show that if $S_1, S_2, \ldots, S_\atkrbudget \in \calS$ are chosen such that for each $i$, $\abs{S_i} \geq \beta \max_{S} w(S \setminus (S_1 \cup \cdots \cup S_{i-1}))$, then $S_1, \ldots, S_\atkrbudget$ gives a $(1 - e^{-\beta})$ approximation for the maximum $\atkrbudget$-coverage problem. Since \greedyatk\ computes $\atkrset$ precisely in this greedy manner, the following lemma is an immediate consequence of a of Hochbaum and Pathria's main result.

\begin{lemma}[cf.\ Theorem~1 in~\cite{Hochbaum1998analysis}]
  \label{lem:greedy_bound}
  Suppose $k = 1$ and each call to \chg\ in \greedyatk\ returns a single-point attack whose score is a $\beta$-fraction of the optimal single point attack ($0 \leq \beta \leq 1$). Then the score of the $\atkrset$ returned by \greedyatk\ is a $(1 - e^{-\beta})$-fraction of the optimal $\atkrbudget$-point attack's score.
\end{lemma}

The lemma establishes Theorem~\ref{thm:greedy_bound_general} for the case $k = 1$. We now establish Theorem~\ref{thm:greedy_bound_general} for general $k$. The subtlety arises because it may no longer be the case that placing a single attack point $(\x, y) \in \IR(\x')$ changes $\learner$'s classification of $\x'$. However, placing $(\x, y)$ with multiplicity $k' = \lceil k / 2 \rceil$ will always change $\learner$'s classification (assuming $\learner[\x'; D] \neq y$). We assume that all points in $\evalpool$ have the same target label, $y$, so that every possible attack point has a non-negative contribution to $\atkr$'s score. Given an attack $\atkrset$, we define
\begin{align*}
  \tsi(\atkrset) &= \sum_{(\x, y) \in \evalpool} \mathbf{1}_{(\learner[\x; \trset \cup \atkrset] = y)} - \mathbf{1}_{(\learner[\x; \trset] = y)}\\
  \coverage(\atkrset) &= \sum_{(\x, y) \in \evalpool} \mathbf{1}_{\IR(\x) \cap \atkrset \neq \varnothing} \cdot w(\x).
\end{align*}
Again, by the assumption that all $(\x, y') \in \evalpool$ have $y' = y$, we have $w(\x) \geq 0$ for all $\x$. Similarly, we have
\[
\mathbf{1}_{(\learner[\x; \trset \cup \atkrset] = y)} - \mathbf{1}_{(\learner[\x; \trset] = y)} \geq 0.
\]
Additionally, we have
\[
\mathbf{1}_{(\learner[\x; \trset \cup \atkrset] = y)} - \mathbf{1}_{(\learner[\x; \trset] = y)} \leq \mathbf{1}_{\IR(\x) \cap \atkrset \neq \varnothing}w(\x).
\]
To see this, consider $(\x, y) \in \evalpool$ with $\learner[\x; \trset] \neq y'$. Then $\IR(\x) \cap \atkrset \neq \varnothing$ is necessary and sufficient to have $w(\x) = 1$. This condition is also necessary, though not generally sufficient, to have $\mathbf{1}_{(\learner[\x; \trset \cup \atkrset] = y)} - \mathbf{1}_{(\learner[\x; \trset] = y)} = 1$ because adding fewer than $k'$ points in $\IR(\x)$ to $\trset$ may not change $\learner[\x; \trset]$. Thus, in general, we have
\begin{equation}
  \label{eqn:score-coverage}
  \tsi(\atkrset) \leq \coverage(\atkrset) \quad\text{for all}\quad \atkrset.
\end{equation}

We note that for $k = 1$, we do have $\tsi(\atkrset) = \coverage(\atkrset)$. More generally, if $\atkrset$ has the property that all points in $\atkrset$ appear with sufficient multiplicity (that is at most $k'$), then we can also guarantee that $\tsi(\atkrset) = \coverage(\atkrset)$. Since \greedyatk\ always adds points to $\atkrset$ with sufficient multiplicity to change classifications, we have the following.

\begin{lemma}
  \label{obs:greedy}
  For any attack $\atkrset$ returned by \greedyatk, we have
  \[
  \tsi(\atkrset) = \coverage(\atkrset).
  \]
\end{lemma}

We also note that the number of \emph{distinct} points in $\atkrset$ is at least $\atkrbudget' = \lfloor \atkrbudget / k' \rfloor$. Since the distinct points are chosen greedily by \greedyatk, we can apply the analysis of~\cite{Hochbaum1998analysis} to obtain the following lemma.

\begin{lemma}
  \label{lem:greedy-coverage}
  Let $\atkrbudget' = \lfloor \atkrbudget / k' \rfloor$, and let $\atkcov'$ be a maximal coverage $\atkrbudget'$-point \emph{coverage} set. That is,
  \[
  \coverage(\atkcov') = \max_{\atkrset : \abs{\atkrset} = \atkrbudget'} \coverage(\atkrset).
  \]
  and let $\atkrset^*$ be an attack returned by \greedyatk during an execution in which each call to \chg\ returns a point covering a $\beta$-fraction of the optimal coverage single point attack. Then
  \[
  \coverage(\atkcov') \leq (1 - 1/e^{\beta})^{-1} \coverage(\atkrset^*).
  \]
\end{lemma}

We now have all of the prerequisites to prove Theorem~\ref{thm:greedy_bound_general}.

\begin{proof}[Proof of Theorem~\ref{thm:greedy_bound_general}]
  Consider an optimal $\atkrbudget$-point attack $\atkopt$. That is, an attack that satisfies
  \[
  \tsi(\atkopt) = \max_{\atkrset : \abs{\atkrset} = \atkrbudget} \tsi(\atkrset).
  \]
  Define $k' = \lceil k / 2 \rceil$ and $\atkrbudget' = \lfloor \atkrbudget / k' \rfloor$, and let be an optimal cover  $\atkcov$ be an optimal \emph{coverage} $b'$-point attack. That is,
  \[
  \coverage(\atkcov) = \max_{\atkrset : \abs{\atkrset} = \atkrbudget'} \coverage(\atkrset).
  \]
  Finally, let $\atkrset^*$ be an attack returned by \greedyatk. Then we have
  \begin{align}
    \tsi(\atkopt) &\leq \coverage(\atkopt) \label{eqn:opt-coverage}\\
    &\leq k' \coverage(\atkcov)\label{eqn:sub-additive}\\
    &\leq k' (1 - 1/e)^{-1} \coverage(\atkrset^*)\label{eqn:partial-coverage}\\
    &= k' (1 - 1/e)^{-1} \tsi(\atkrset^*)\label{eqn:partial-score}.
  \end{align}
  Combining these expressions, we get
  \[
  \frac{1}{k} (1 - 1/e^\beta) \tsi(\atkopt) \leq \tsi(\atkrset^*),
  \]
  which gives the desired result. The justifications of the computations above are as follows: (\ref{eqn:opt-coverage}) follows from~(\ref{eqn:score-coverage}). Equation~(\ref{eqn:sub-additive}) follows from the sub-additivity of $\coverage$: for any sets $A$ and $B$, $\coverage(A \cup B) \leq \coverage(A) + \coverage(B)$. Partition $\atkopt$ into $k'$ sets of size at most $\atkrbudget'$, say $\atkopt = A_1 \cup A_2 \cup \cdots \cup A_{k'}$. Then we have $\coverage(A_i) \leq \coverage(\atkcov)$ for all $i$, whence~(\ref{eqn:sub-additive}) follows. Equation~(\ref{eqn:partial-coverage}) follows from Lemma~\ref{lem:greedy-coverage}, and Equation~(\ref{eqn:partial-score}) follows from Equation~(\ref{eqn:score-coverage}).
\end{proof}

\subsection*{An additional optimization}

We offer the following practical optimization to \chg~ (noting that it may, of course, still take exponential time).
We replace line~\ref{line:shortcircuit} in Algorithm~\ref{alg:chg} with the following:
\begin{align}
\wedge \left(|\potentialedge| - 1  \geq d + 1 \vee \text{(\ref{eqn:qclp}) has a solution}\right)
\end{align}
Then, instead of maintaining the point \(\mathbf x\), we store the hyper-edge \(\mathbf s\) and solve the QCLP only to find the ultimate attack point.
By Helly's theorem and the fact that hyperspheres are convex, if all subsets of \(\potentialedge\) of size \(d+1\) are in \(E\) (have a point of mutual intersection), then \(\potentialedge\) itself is a hyper-edge and there is no need to invoke the QCLP solver.
Note that in checking all subsets of size \(|\potentialedge| - 1 > d + 1\), we are implicitly checking all subsets of size \(d+1\).
Therefore, when \(|\potentialedge| > d+1\) then this condition on both necessary \emph{ and sufficient}, thus there is no need to check if~(\ref{eqn:qclp}) has a solution.

Note that this optimization is not possible when \chg~ is not run to completion.
While Helly's theorem guarantees the existence of a region of mutual overlap, it provides no witness (no point within that region).
Such a witness is necessary to compute the \tlr~ when terminating before all hyper-edges are constructed (but \tlr~ can be computed given a maximal hyper-edge directly).

%% file: app_exps.tex
Tables~\ref{tab:normal}, \ref{tab:uniform} show \atkr's score (out of 10) for various dimensions as training set sizes.
\begin{table}[h]
\centering
\begin{tabular}{|l||r|r|r|r|r|}
\hline
\(d\) &   \(|\trset|=8\)  &   \(|\trset|=16\)  &   \(|\trset|=32\)  &   \(|\trset|=64\)  &   \(|\trset|=128\) \\\hline
\hline
 2  &   3.7 (3.5) &   3.3 (3.3) &   2.6 (2.5) &   2.0 (2.2) &   2.0 (2.0) \\\hline
 4  &   4.7 (4.8) &   4.6 (4.9) &   3.8 (3.7) &   2.7 (2.8) &   2.6 (2.4) \\\hline
 8  &   7.6 (7.0) &   6.2 (5.7) &   5.3 (5.0) &   3.9 (3.9) &   3.6 (3.9) \\\hline
 16 &   9.7 (9.1) &   9.5 (8.3) &   8.9 (7.6) &   7.5 (6.1) &   6.1 (5.8) \\\hline
 32 &  10.0 (9.7) &  10.0 (9.1) &  10.0 (9.1) &  10.0 (8.4) &  10.0 (8.0) \\\hline
\end{tabular}
\caption{{\bf Results on \texttt{Normal}}. Scores of optimal
  one point attack for experiments described in Section~\ref{synthexp} 
  under $\ell_2$ norm ($\ell_\infty$ norm).
  Instances are constructed by sampling a training set of $|\trset|$ points in $d$ dimensions from
  $\mathrm{Normal}(\mathbf{0}, \mathbf{I}_d)$ and sampling 10 evaluation points
  from the same distribution. Reported scores are a mean over 10 instances. As
  dimensionality rises, the score of the optimal attack rises. As training set
  sizes rise, the score of the optimal attack falls in all but \(d=32\) (where it achieves the maximum possible score of 10).}
\label{tab:normal}
\end{table}

\begin{table}[h]
\centering
\begin{tabular}{|l||r|r|r|r|r|}
\hline
\(d\) &   \(|\trset|=8\)  &   \(|\trset|=16\)  &   \(|\trset|=32\)  &   \(|\trset|=64\)  &   \(|\trset|=128\) \\ \hline
\hline
 2  &   4.2 (4.2) &   3.5 (3.3) &   2.4 (2.5) &   1.9 (1.9) &   1.5 (1.6) \\\hline
 4  &   5.2 (5.3) &   3.9 (4.3) &   3.4 (3.2) &   2.4 (2.8) &   2.2 (2.2) \\\hline
 8  &   8.3 (8.9) &   6.9 (8.0) &   4.9 (5.8) &   4.0 (4.9) &   3.1 (2.9) \\\hline
 16 &   9.9 (9.8) &   9.2 (9.9) &   9.1 (9.8) &   7.6 (9.4) &   6.0 (8.1) \\\hline
 32 &  10.0 (10.0) &  10.0 (10.0) &  10.0 (10.0) &  10.0 (10.0)&  10.0 (10.0) \\\hline
\end{tabular}
\caption{{\bf Results on \texttt{Uniform}}. Scores of optimal
  one point attack for experiments described in Section~\ref{synthexp} 
  under $\ell_2$ norm ($\ell_\infty$ norm).
  Instances are constructed by sampling a training set of $|\trset|$ points in $d$ dimensions from
  $\mathrm{Uniform}([0,1]^d)$ and sampling 10 evaluation points
  from the same distribution. Reported scores are a mean over 10 instances.
  Results are qualitative similar to those in Table~\ref{tab:normal} in that
  increasing dimensionality increases the score of the optimal attack while
  increasing training set size reduces it.} 
\label{tab:uniform}
\end{table}

\begin{table}[h]
\centering
\begin{tabular}{|l||r|r|r|r|r|}
\hline
\(d\) &   \(|\trset|=8\)  &   \(|\trset|=16\)  &   \(|\trset|=32\)  &   \(|\trset|=64\)  &   \(|\trset|=128\) \\\hline
\hline
 2  &  0.300 (0.224) &  0.213 (0.300) &  0.163 (0.224) &  0.000 (0.133) &  0.149 (0.149) \\\hline
 4  &  0.213 (0.200) &  0.340 (0.314) &  0.359 (0.367) &  0.260 (0.327) &  0.340 (0.340) \\\hline
 8  &  0.400 (0.211) &  0.200 (0.396) &  0.396 (0.422) &  0.100 (0.180) &  0.221 (0.348) \\\hline
 16 &  0.153 (0.233) &  0.224 (0.213) &  0.233 (0.267) &  0.453 (0.314) &  0.233 (0.291) \\\hline
 32 &  0.000 (0.153) &  0.000 (0.233) &  0.000 (0.180) &  0.000 (0.306) &  0.000 (0.333) \\\hline
\end{tabular}
\caption{{\bf Standard Error of the Mean Values for Table~\ref{tab:normal}.}}
\end{table}

\begin{table}[h]
\centering
\begin{tabular}{|l||r|r|r|r|r|}
\hline
\(d\) &   \(|\trset|=8\)  &   \(|\trset|=16\)  &   \(|\trset|=32\)  &   \(|\trset|=64\)  &   \(|\trset|=128\) \\\hline
\hline
 2  &  0.200 (0.327) &  0.269 (0.213) &  0.163 (0.224) &  0.100 (0.100) &  0.167 (0.167) \\\hline
 4  &  0.249 (0.396) &  0.314 (0.300) &  0.221 (0.249) &  0.163 (0.200) &  0.133 (0.133) \\\hline
 8  &  0.335 (0.314) &  0.504 (0.333) &  0.277 (0.533) &  0.258 (0.314) &  0.180 (0.180) \\\hline
 16 &  0.100 (0.133) &  0.200 (0.100) &  0.277 (0.133) &  0.306 (0.221) &  0.298 (0.379) \\\hline
 32 &  0.000 (0.000) &  0.000 (0.000) &  0.000 (0.000) &  0.000 (0.000) &  0.000 (0.000) \\\hline
\end{tabular}
\caption{{\bf Standard Error of the Mean  Values for Table~\ref{tab:uniform}.}}
\end{table}

Figure~\ref{fig:pdmadelon} plots the attacker's effectiveness for two additional
datasets, \texttt{Madelon}, an artificial dataset constructed as part of a NIPS
2003 challenge (\cite{guyon2004result}), and \texttt{Parkinson}, a dataset of
sound recordings from 20 Parkinson's Disease patients and 20 healthy subjects
(\cite{sakar2013collection}). Tables~\ref{tab:pdmadelon} reports zero-one losses
and the proportion of variance explained.

\begin{figure}[t]
  \centering
  \begin{tabular}{cc}
    \includegraphics[width=0.5\columnwidth]{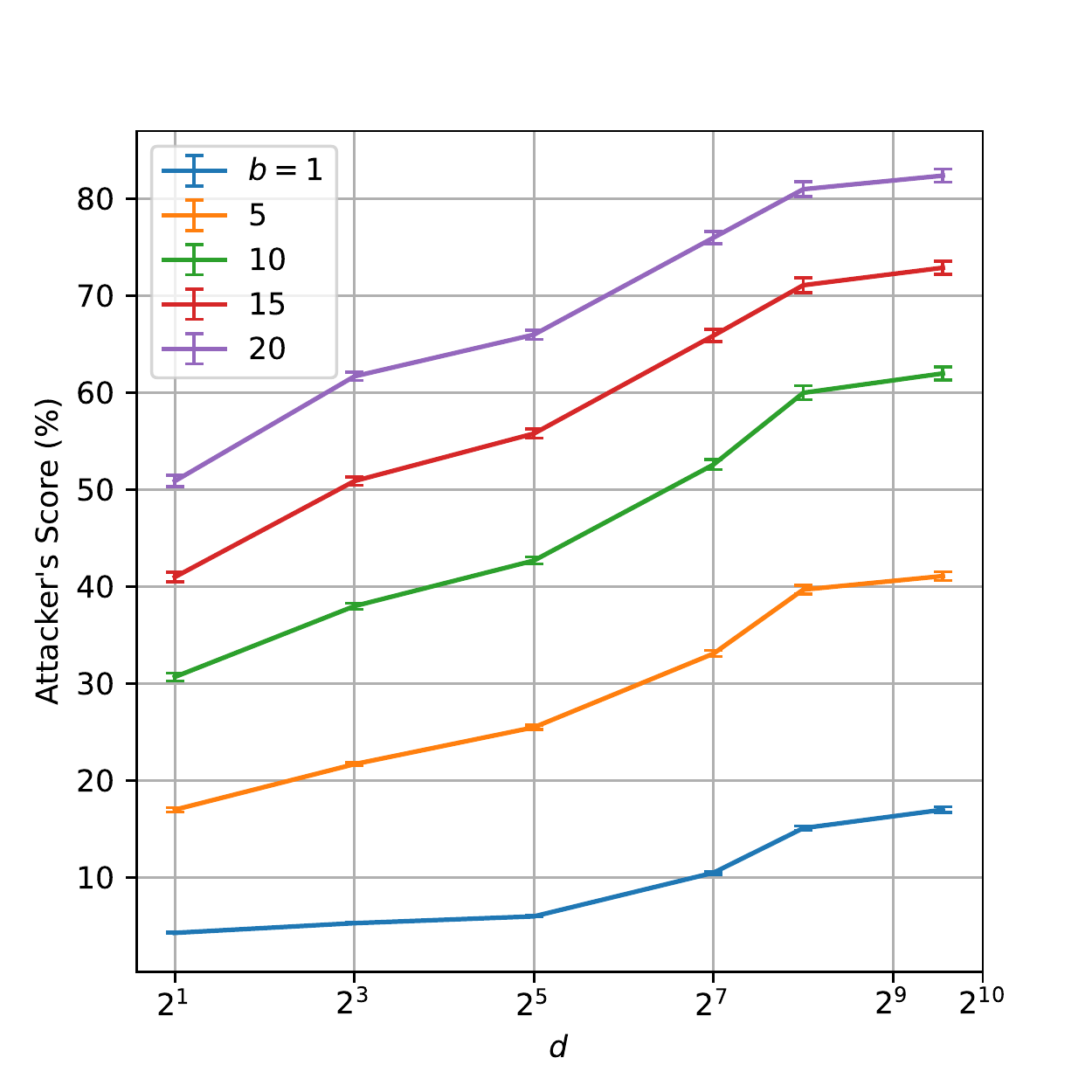} & \includegraphics[width=0.5\columnwidth]{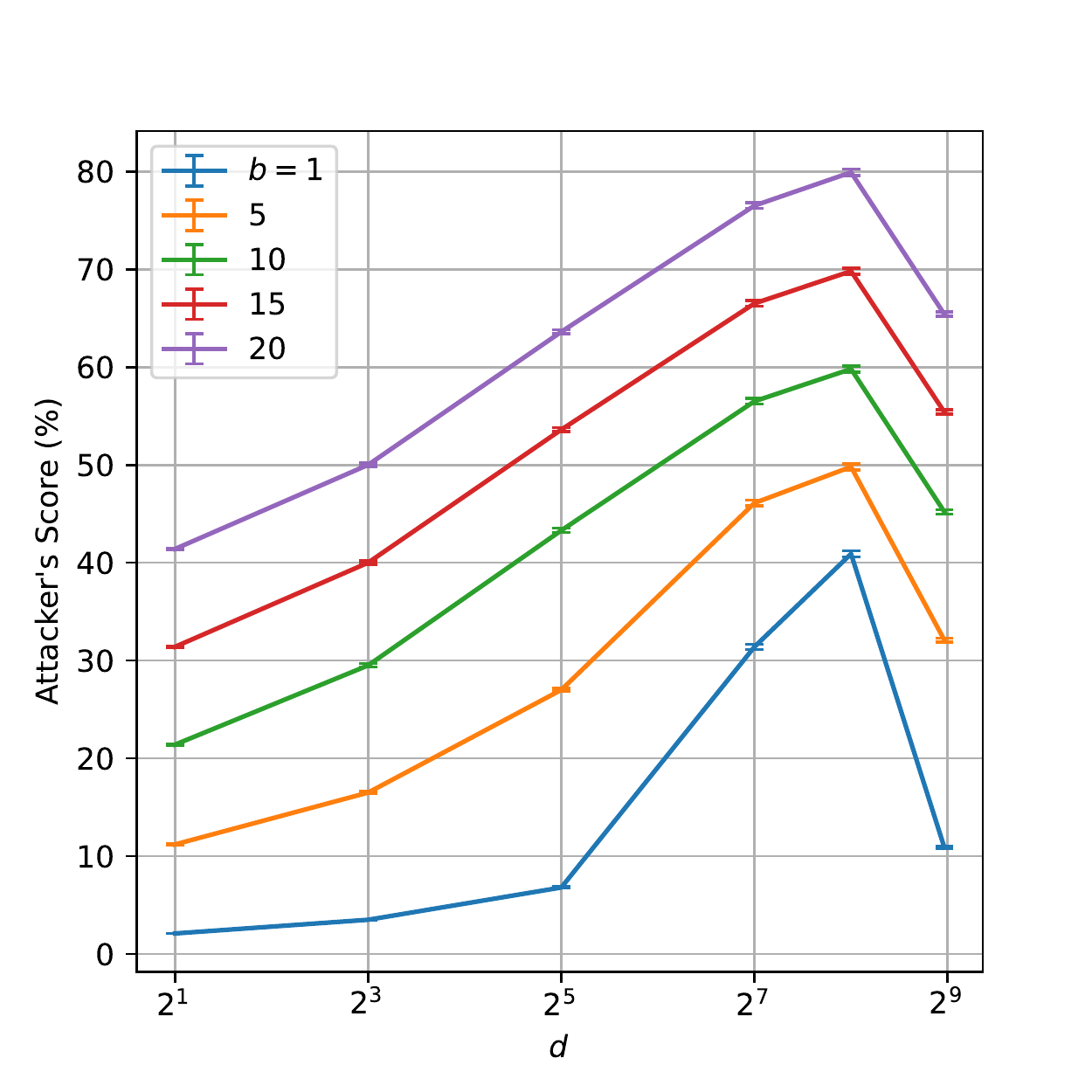} \\
    \includegraphics[width=0.5\columnwidth]{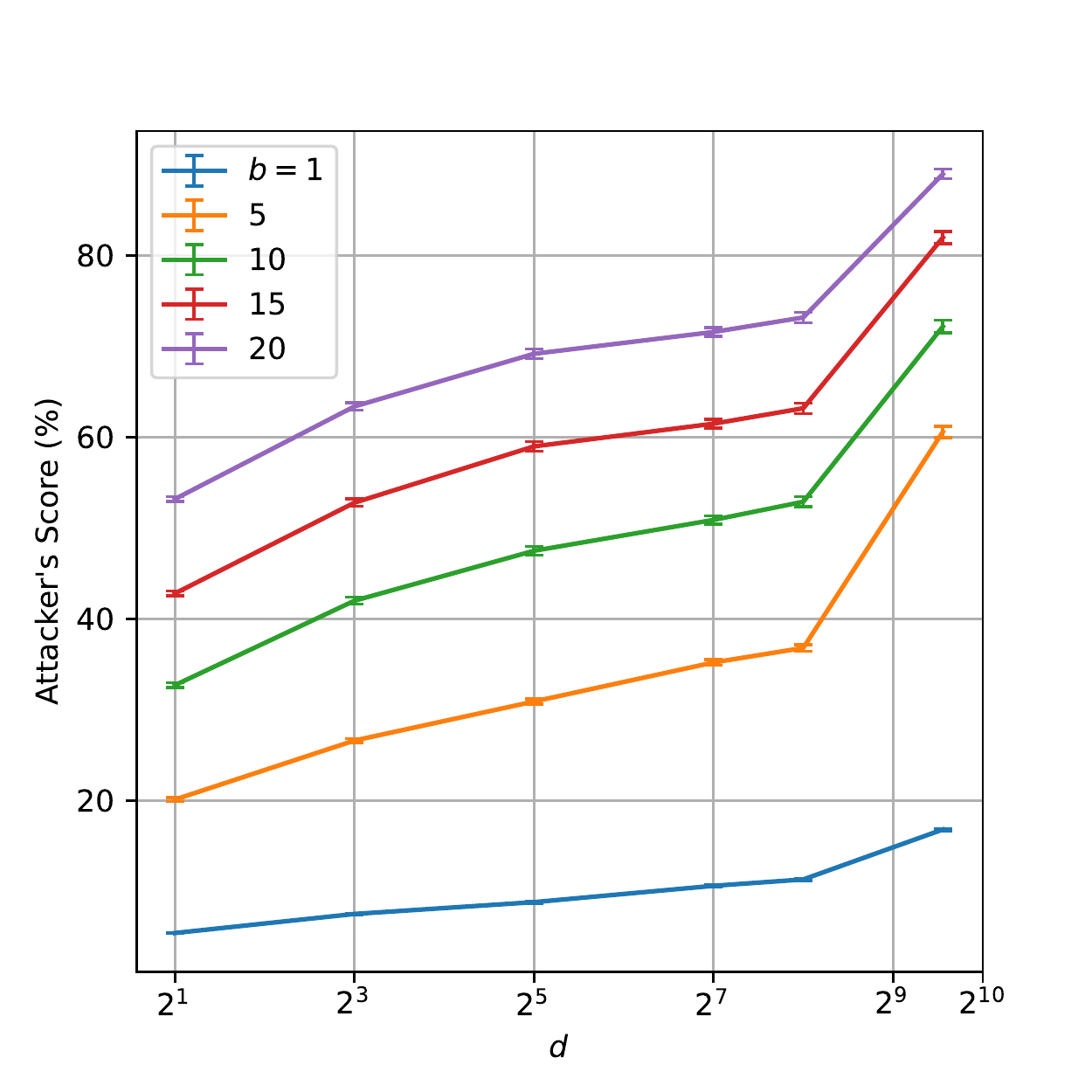} & \includegraphics[width=0.5\columnwidth]{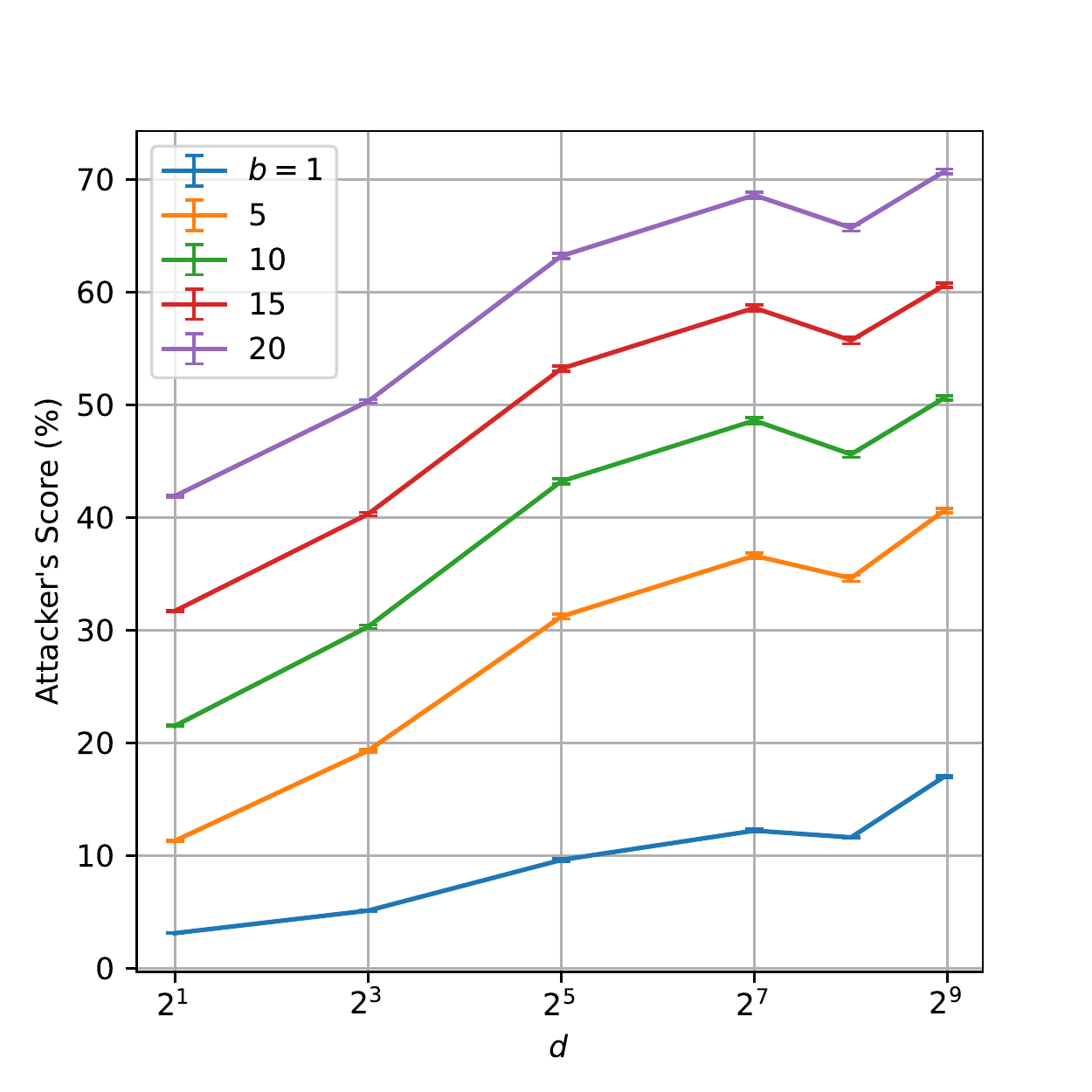}
  \end{tabular}
  \caption{{\bf Attacker's Effectiveness as a Function of Dimension on
      Additional Datasets.} On the left is the \texttt{Madelon} dataset
    (available at https://archive.ics.uci.edu/ml/datasets/Madelon), while on the
    right is the \texttt{Parkinson} dataset (available at
    https://archive.ics.uci.edu/ml/datasets/parkinsons). Top row is under
    $\ell_2$-norm while bottom row is under $\ell_\infty$-norm. Dimensionality was reduced
    via PCA up to various dimensions. We plot \atkr's score (as a percentage of
    her total evaluation pool) averaged over 10 trials for each label, for various attacker
    budgets. Error bars denote standard error of the mean. }
  \label{fig:pdmadelon}
\end{figure}

\begin{figure}[t]
  \centering
  \begin{tabular}{cc}
    \includegraphics[width=0.5\columnwidth]{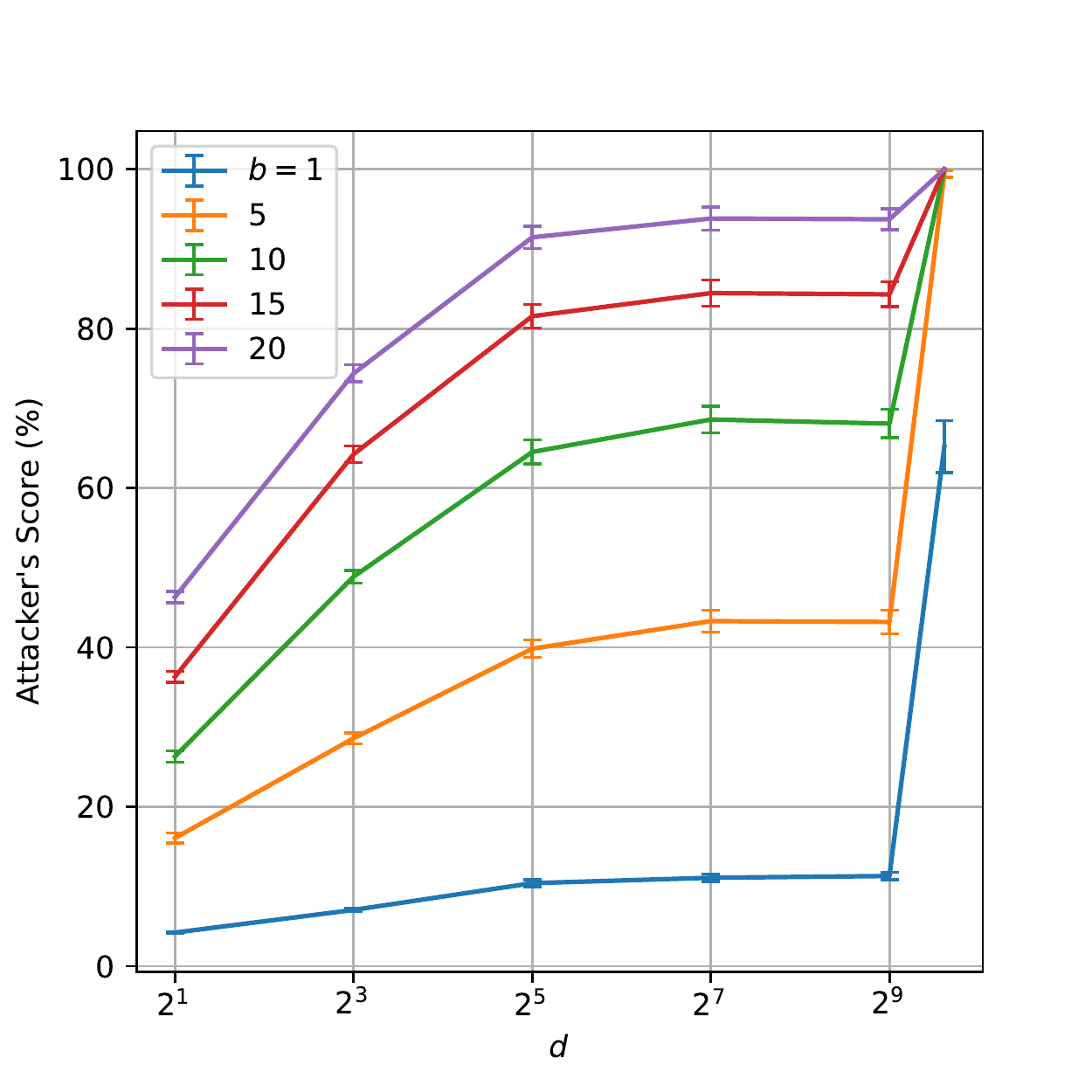} & \includegraphics[width=0.5\columnwidth]{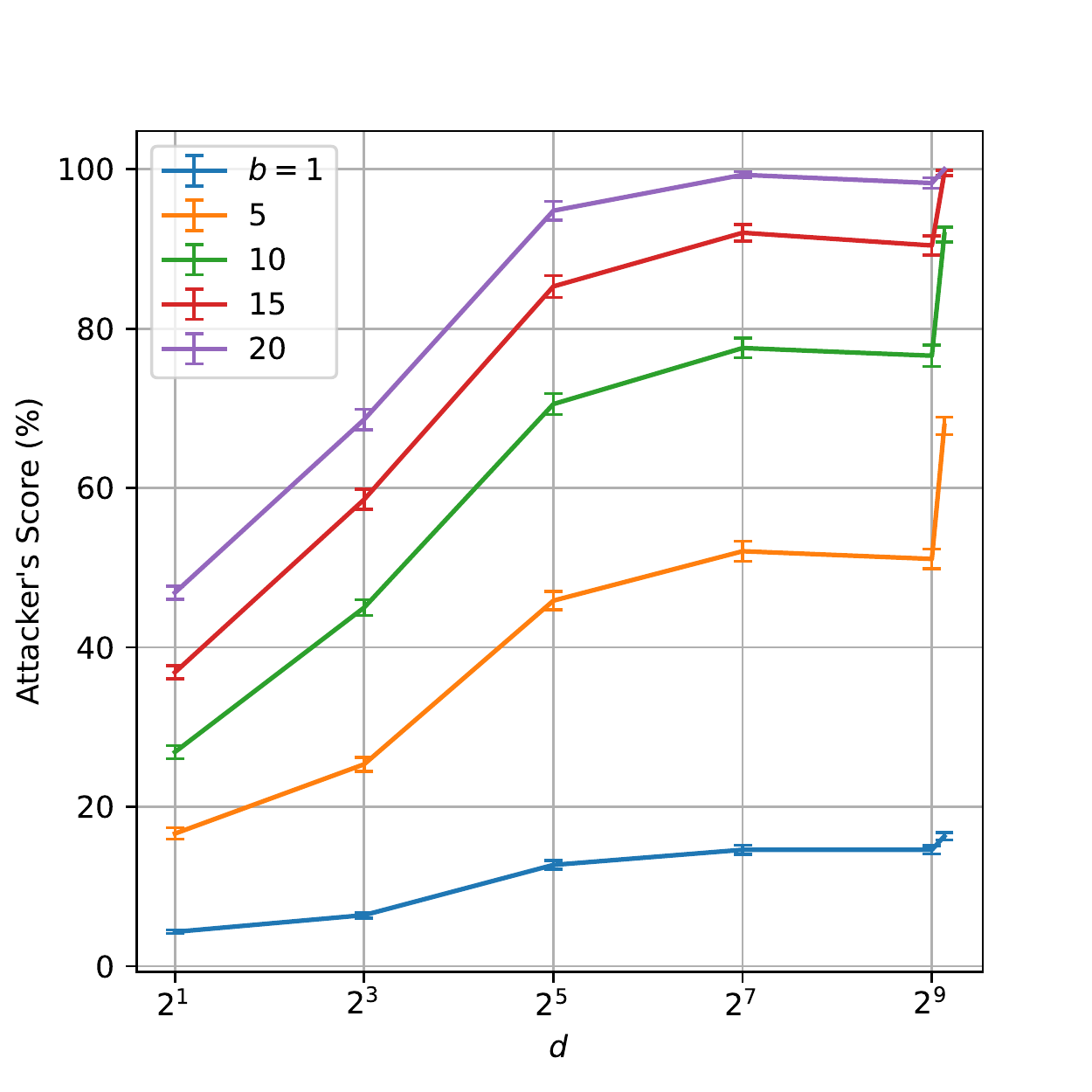} \\
  \end{tabular}
  \caption{{\bf Attacker's Effectiveness as a Function of Dimension on of MNIST
      and HAPT Datasets under $\ell_\infty$-norm.} On the left is the \texttt{MNIST}
    dataset, on the right is the \texttt{HAPT} dataset. Dimensionality was
    reduced via PCA up to various dimensions. We plot \atkr's score (as a
    percentage of her total evaluation pool) averaged over 10 trials for each label, for
    various attacker budgets. Error bars denote standard error of the mean. }
  \label{fig:mnisthaptinf}
\end{figure}

\begin{SCtable}
  \centering
  \ifsqueeze \small \fi
\begin{tabular}{|c|c|c|c|c|}\hline
                                & \multicolumn{2}{c|}{Madelon} & \multicolumn{2}{c|}{Parkinson} \\ \hline
\multicolumn{1}{|c|}{\(d\)} & \texttt{HOL}     & \texttt{VarE}             & \texttt{HOL}         & \texttt{VarE}            \\ \hline\hline
\multicolumn{1}{|c|}{Original}  & 0.347 (0.430)           & 1.000                     & 0.114 (0.291)               & 1.000          \\ \hline 
\multicolumn{1}{|c|}{128}       & 0.346 (0.360)           & 0.730                     & 0.123 (0.162)               & 0.913     \\ \hline
\multicolumn{1}{|c|}{32}        & 0.267 (0.317)           & 0.412                     & 0.152 (0.234)               & 0.702     \\ \hline
\multicolumn{1}{|c|}{8}         & 0.180 (0.203)           & 0.294                     & 0.194 (0.173)               & 0.462      \\ \hline
\multicolumn{1}{|c|}{2}         & 0.416 (0.429)           & 0.189                     & 0.256 (0.280)               & 0.223     \\ \hline
\end{tabular}
\caption{{\bf Loss and Variance Explained for Additional Datasets.} We report the zero-one loss on a held out dataset (\texttt{HOL}) and the proportion of variance explained (\texttt{VarE}) for the each dataset (the original and those with dimension reduced by PCA).}
\label{tab:pdmadelon}
\end{SCtable}